%% file: tree15.tex
\documentclass[letterpaper]{article} 
\usepackage{nips15submit_e,times}
\usepackage{chapterbib}
\usepackage{url}
\usepackage{floatrow}
\usepackage{amsmath,amssymb, stmaryrd,enumerate,bbm,bm}
\usepackage{amsthm}
\usepackage{algorithm}
\usepackage{algpseudocode}
\usepackage[hidelinks]{hyperref}
\usepackage{booktabs}
\usepackage{wrapfig}
\usepackage{multirow}
\usepackage{graphicx}
\usepackage{subfig}
\usepackage{xr}
\usepackage[normalem]{ulem}

\newtheorem{theorem}{Theorem}
\newtheorem{corollary}{Corollary}
\newtheorem{lemma}{Lemma}
\newtheorem{Prop}{Proposition}

\newfloatcommand{capbtabbox}{table}[][\FBwidth]

\newcommand{\tmmathbf}[1]{\ensuremath{\boldsymbol{#1}}}
\newcommand{\tmop}[1]{\ensuremath{\operatorname{#1}}}

\DeclareMathOperator*{\argmin}{arg\,min}
\DeclareMathOperator*{\argmax}{arg\,max}

\title{Parallelizing MCMC with Random Partition Trees}

\author{
Xiangyu Wang \\
Dept. of Statistical Science\\
Duke University\\
\texttt{xw56@stat.duke.edu} \\
\And
Fangjian Guo\\
Dept. of Computer Science\\
Duke University\\
\texttt{guo@cs.duke.edu} \\
\And
Katherine A.~Heller\\
Dept. of Statistical Science\\
Duke University\\
\texttt{kheller@stat.duke.edu} \\
\And
David B.~Dunson\\
Dept. of Statistical Science\\
Duke University\\
\texttt{dunson@stat.duke.edu} \\
}

\nipsfinalcopy 

\begin{document}

\maketitle

\begin{abstract}
The modern scale of data has brought new challenges to Bayesian inference. In particular, conventional MCMC algorithms are computationally very expensive for large data sets.  A promising approach to solve this problem is embarrassingly parallel MCMC (EP-MCMC), which first partitions the data into multiple subsets and runs independent sampling algorithms on each subset. The subset posterior draws are then aggregated via some combining rules to obtain the final approximation. Existing EP-MCMC algorithms are  limited by approximation accuracy and difficulty in resampling. In this article, we propose a new EP-MCMC algorithm \emph{PART} that solves these problems. The new algorithm applies {\em random partition trees} to combine the subset posterior draws, which is distribution-free, easy to resample from and can adapt to multiple scales. We provide theoretical justification and extensive experiments illustrating empirical performance.
\end{abstract}

\section{Introduction}
Bayesian methods are popular for their success in analyzing complex data sets. However,  for large data sets, Markov Chain Monte Carlo (MCMC) algorithms, widely used in Bayesian inference, can suffer from huge computational expense.  With large data, there is increasing time per iteration, increasing time to convergence, and difficulties with processing the full data on a single machine due to memory limits. 
 To ameliorate these concerns, various methods such as stochastic gradient Monte Carlo \cite{welling2011bayesian} and sub-sampling based Monte Carlo \cite{maclaurinfirefly} have been proposed. Among directions that have been explored, embarrassingly parallel MCMC (EP-MCMC) seems  most promising. EP-MCMC algorithms typically divide the data into multiple subsets and run independent MCMC chains simultaneously on each subset. The posterior draws are then aggregated according to some rules to produce the final approximation. This approach is clearly more efficient as now each chain involves a much smaller data set and the sampling is communication-free. The key to a successful EP-MCMC algorithm lies in the speed and accuracy of the combining rule.

Existing EP-MCMC algorithms can be roughly divided into three categories. The first relies on asymptotic normality of posterior distributions. \cite{scott2013bayes} propose a ``Consensus Monte Carlo'' algorithm, which produces final approximation by a weighted averaging over all subset draws. This approach is effective when the posterior distributions are close to Gaussian, but could suffer from huge bias when skewness and multi-modes are present. The second category relies on calculating an appropriate variant of a mean or median of the subset posterior measures \cite{minsker2014scalable,srivastavawasp}.  These approaches rely on asymptotics (size of data increasing to infinity) to justify accuracy, and lack guarantees in finite samples.
The third category relies on the \emph{product density equation (PDE)} in (\ref{eq:pde}). Assuming $X$ is the observed data and $\theta$ is the parameter of interest, when the observations are iid conditioned on $\theta$, for any partition of $X = X^{(1)}\cup X^{(2)}\cup\cdots\cup X^{(m)}$, the following identity holds,
\begin{align}
p(\theta|X)\propto \pi(\theta)p(X | \theta) \propto p(\theta|X^{(1)})p(\theta|X^{(2)})\cdots p(\theta|X^{(m)}), \label{eq:pde}
\end{align}
if the prior on the full data and subsets satisfy $\pi(\theta) = \prod_{i = 1}^m \pi_i(\theta)$. \cite{neiswanger2013asymptotically} proposes using kernel density estimation on each subset posterior and then combining via \eqref{eq:pde}. They use an independent Metropolis sampler to resample from the combined density. \cite{wang2013parallel} apply the Weierstrass transform directly to \eqref{eq:pde} and developed two sampling algorithms based on the transformed density. These methods guarantee the approximation density converges to the true posterior density as the number of posterior draws increase. However, as both are kernel-based, the two methods are limited by two major drawbacks. The first is the inefficiency of resampling. Kernel density estimators are essentially mixture distributions. Assuming we have collected 10,000 posterior samples on each machine, then multiplying just two densities already yields a mixture distribution containing $10^8$ components, each of which is associated with a different weight. The resampling requires the independent Metropolis sampler to search over an exponential number of mixture components and it is likely to get stuck at one ``good'' component, resulting in high rejection rates and slow mixing. The second is the sensitivity to bandwidth choice, with one bandwidth applied to the whole space.

In this article, we propose a novel EP-MCMC algorithm termed ``parallel aggregation random trees'' (\emph{PART}), which solves the above two problems. The algorithm inhibits the explosion of mixture components so that the aggregated density is easy to resample. In addition, the density estimator is able to adapt to multiple scales and thus achieve better approximation accuracy.  In Section 2, we motivate the new methodology and present the algorithm. In Section 3, we present error bounds and prove consistency of \emph{PART} in the number of posterior draws. Experimental results are presented in Section 4. Proofs and part of the numerical results are provided in the supplementary materials.

\section{Method}
Recall the \emph{PDE} identity \eqref{eq:pde} in the introduction. When data set $X$ is partitioned into $m$ subsets $X = X^{(1)}\cup\cdots\cup X^{(m)} $, the posterior distribution of the $i^\text{th}$ subset can be written as
\begin{equation}
f^{(i)}(\theta) \propto \pi(\theta)^{1/m} p(X^{(i)} | \theta),
\end{equation}
where $\pi(\theta)$ is the prior assigned to the full data set. Assuming observations are iid given $\theta$, the relationship between the full data posterior and subset posteriors is captured by 
\begin{equation}
p(\theta |X) \propto \pi(\theta) \prod_{i=1}^m p(X^{(i)}| \theta) \propto \prod_{i=1}^m f^{(i)}(\theta)\label{eqs:multiplcation}.
\end{equation}
Due to the flaws of applying kernel-based density estimation to  \eqref{eqs:multiplcation} mentioned above, we propose to use \emph{random partition trees} or {\em multi-scale histograms}. Let $\mathcal{F_K}$ be the collection of all $\mathbb{R}^p$-partitions formed by $K$ disjoint rectangular blocks, where a rectangular block takes the form of $A_k \overset{def}{=} (l_{k,1}, r_{k,1}] \times (l_{k,2}, r_{k,2}] \times \cdots (l_{k,p}, r_{k,p}]\subseteq \mathbb{R}^p$ for some $l_{k,q} < r_{k,q}$. A $K$-block histogram is then defined as
\begin{equation}
\hat{f}^{(i)}(\theta) = \sum_{k=1}^{K} \frac{n_k^{(i)}}{N|A_k|} \bm{1}(\theta \in A_k), \label{eqs:subset-estimator}
\end{equation}
where $\{A_k: k=1,2,\cdots,K\}\in\mathcal{F_K}$ are the blocks and $N, n_k^{(i)}$ are the total number of posterior samples on the $i^{\text{th}}$ subset and of those inside the block $A_k$ respectively (assuming the same $N$ across subsets). We use $|\cdot|$ to denote the area of a block. Assuming each subset posterior is approximated by a $K$-block histogram, if the partition $\{A_k\}$ is 
restricted to be \textit{the same} across all subsets, then the aggregated density after applying \eqref{eqs:multiplcation} is still a $K$-block histogram (illustrated in the supplement),
\begin{equation}
\hat p(\theta | X) = \frac{1}{Z} \prod_{i=1}^m \hat{f}^{(i)}(\theta) = \frac{1}{Z} \sum_{k=1}^K \bigg(\prod_{i=1}^m \frac{n_k^{(i)}}{|A_k|} \bigg) \bm{1}(\theta \in A_k) = \sum_{k=1}^K w_k g_k(\theta), \label{eqs:full-est} 
\end{equation}
where $Z=\sum_{k=1}^K \prod_{i=1}^m n_k^{(i)} / |A_k|^{m-1}$ is the normalizing constant, $w_k$'s are the updated weights, and $g_k(\theta) = \text{unif}(\theta;A_k)$ is the block-wise distribution. Common histogram blocks across subsets control the number of mixture components, leading to simple aggregation and resampling procedures.  Our \emph{PART} algorithm consists of {\em space partitioning} followed by {\em density aggregation}, with aggregation simply multiplying densities across subsets for each block and then normalizing.

\subsection{Space Partitioning}
To find good partitions, our algorithm recursively bisects (not necessarily evenly) a previous block along a randomly selected dimension, subject to certain rules. Such partitioning is multi-scale and related to wavelets \cite{liu2014multivariate}. Assume we are currently splitting the block $A$ along the dimension $q$ and denote the posterior samples in $A$ by $\{\theta^{(i)}_j\}_{j\in {A}}$ for the $i^\text{th}$ subset. The cut point on dimension $q$ is determined by a partition rule $\phi(\{\theta_{j,q}^{(1)}\}, \{\theta_{j,q}^{(2)}\}, \cdots, \{\theta_{j,q}^{(m)}\})$. The resulting two blocks are subject to further bisecting under the same procedure until one of the following stopping criteria is met --- (i) $n_k/N < \delta_{\rho}$ or (ii) the area of the block $|A_k|$ becomes smaller than $\delta_{|A|}$. The algorithm returns a tree with $K$ leafs, each corresponding to a block $A_k$. Details are provided in Algorithm \ref{alg:tree}.

\begin{algorithm}
\caption{Partition tree algorithm}
\label{alg:tree}
\begin{algorithmic}[1]
\Procedure{BuildTree}{$\{\theta_j^{(1)}\}, \{\theta_j^{(2)}\}, \cdots, \{\theta_j^{(m)}\}$, $\phi(\cdot)$, $\delta_\rho$, $\delta_a$, $N$, $L$, $R$}
\State $D \leftarrow \{1,2,\cdots, p\}$ 
\While {$D$ not empty}
\State Draw $q$ uniformly at random from $D$. \Comment{Randomly choose the dimension to cut}
\State $\theta_q^{\ast} \leftarrow \phi(\{\theta_{j,q}^{(1)}\}, \{\theta_{j,q}^{(2)}\}, \cdots, \{\theta_{j,q}^{(m)}\})$, $\quad \mathcal{T}.n^{(i)} \leftarrow \text{Cardinality of } \{\theta_j^{(i)}\}$ for all $i$
\If {$\theta_q^{\ast}-L_q>\delta_a$, $R_q - \theta_q^{\ast} > \delta_a$ \textbf{and} $\min(\sum_{j} \bm{1}(\theta_{j,q}^{(i)} \leq \theta_q^{\ast}), \sum_{j}\bm{1}(\theta_{j,q}^{(i)}> \theta_q^{\ast}))>N\delta_\rho$ for all $i$ }
\State $L' \leftarrow L$, $L_q' \leftarrow \theta_q^{\ast}$, $R' \leftarrow R$, $R_q' \leftarrow \theta_q^{\ast}$ \Comment{Update left and right boundaries}
\State $\mathcal{T.L} \leftarrow \textsc{BuildTree}(\{\theta_j^{(1)}: \theta_{j,q}^{(1)} \leq \theta_q^{\ast}\}, \cdots, \{\theta_j^{(m)}: \theta_{j,q}^{(m)} \leq \theta_{q}^\ast \}, \cdots, N, L, R')$
\State $\mathcal{T.R} \leftarrow \textsc{BuildTree}(\{\theta_j^{(1)}: \theta_{j,q}^{(1)} > \theta_q^{\ast}\}, \cdots, \{\theta_j^{(m)}: \theta_{j,q}^{(m)} > \theta_{q}^\ast \}, \cdots, N, L', R)$
\State \Return $\mathcal{T}$ 
\Else 
\State $D \leftarrow D \setminus \{q\}$ \Comment{Try cutting at another dimension}
\EndIf
\EndWhile
\State $\mathcal{T.L} \leftarrow \text{NULL}$, $\mathcal{T.R} \leftarrow \text{NULL}$, \Return $\mathcal{T}$ \Comment{Leaf node}

\EndProcedure
\end{algorithmic}
\end{algorithm}

In Algorithm \ref{alg:tree}, $\delta_{|A|}$ becomes the minimum edge length of a block $\delta_a$ (possibly different across dimensions). Quantities $L, R \in \mathbb{R}^{p}$ are the left and right boundaries of the samples respectively, which take the sample minimum/maximum when the support is unbounded. We consider two choices for the partition rule $\phi(\cdot)$  --- maximum (empirical) likelihood partition (ML) and median/KD-tree partition (KD). 

\paragraph{Maximum Likelihood Partition (ML)}
ML-partition searches for partitions by greedily maximizing the empirical log likelihood at each iteration. For $m = 1$ we have
\begin{equation}
\theta^{\ast} = \phi_{\text{ML}}(\{\theta_{j,q}, j=1,\cdots,n\}) = \argmax_{n_1 + n_2 = n, A_1\cup A_2 = A} \bigg(\frac{n_1}{n |A_1|}\bigg)^{n_1} \bigg(\frac{n_2}{n |A_2|}\bigg)^{n_2}, \label{eqs:empirical-ml}
\end{equation}
where $n_1$ and $n_2$ are counts of posterior samples in $A_1$ and $A_2$, respectively. The solution to \eqref{eqs:empirical-ml} falls inside the set $\{\theta_j\}$. Thus, a simple linear search after sorting samples suffices (by book-keeping the ordering, sorting the whole block once is enough for the entire procedure). For $m>1$, we have
\begin{equation}
\phi_{q, \text{ML}}(\cdot) = \argmax_{\theta^{\ast} \in \cup_{i=1}^m \{\theta_j^{(i)}\}} \prod_{i=1}^{m} \bigg(\frac{n_1^{(i)}}{n^{(i)}|A_1|} \bigg)^{n_1^{(i)}} \bigg(\frac{n_2^{(i)}}{n^{(i)}|A_2|} \bigg)^{n_2^{(i)}},
\end{equation}similarly solved by a linear search. This is dominated by sorting and takes $O(n \log n)$ time.

\paragraph{Median/KD-Tree Partition (KD)} Median/KD-tree partition cuts at the empirical median of posterior samples. When there are multiple subsets, the median is taken over pooled samples to force $\{A_k\}$ to be the same across subsets. Searching for median takes $O(n)$ time \cite{blum1973time}, which is faster than ML-partition especially when the number of posterior draws is large. The same partitioning strategy is adopted by KD-trees \cite{bentley1975multidimensional}.

\subsection{Density Aggregation}
Given a common partition, Algorithm \ref{alg:onestage} aggregates all subsets \emph{in one stage}. However, assuming a single ``good'' partition for all subsets is overly restrictive when $m$ is large.  Hence, we also consider  \emph{pairwise aggregation}~\cite{neiswanger2013asymptotically,wang2013parallel}, which recursively groups subsets into pairs, combines each pair with Algorithm \ref{alg:onestage}, and repeats until one final set is obtained.  Run time of  \emph{PART} is dominated by space partitioning (\textsc{BuildTree}), with normalization and resampling very fast.

\begin{algorithm}
\caption{Density aggregation algorithm (drawing $N'$ samples from the aggregated posterior)}
\label{alg:onestage}
\begin{algorithmic}[1]
\Procedure{OneStageAggregate}{$\{\theta_j^{(1)}\}, \{\theta_j^{(2)}\}, \cdots, \{\theta_j^{(m)}\}$, $\phi(\cdot)$, $\delta_\rho$, $\delta_a$, $N$, $N'$, $L$, $R$}
\State $\mathcal{T} \leftarrow$ \textsc{BuildTree}($\{\theta_j^{(1)}\}, \{\theta_j^{(2)}\}, \cdots, \{\theta_j^{(m)}\}$, $\phi(\cdot)$, $\delta_\rho$, $\delta_a$, $N$, $L$, $R$), $\quad Z \leftarrow 0$
\State $(\{A_k\}, \{n_k^{(i)}\}) \leftarrow$ \textsc{TraverseLeaf}($\mathcal{T}$)
\For {$k=1,2,\cdots,K$} 
\State $\tilde{w}_k \leftarrow \prod_{i=1}^{m} n_k^{(i)} / |A_k|^{m-1}$, $Z \leftarrow Z + \tilde{w}_k$ \Comment{Multiply inside each block}
\EndFor
\State $w_k \leftarrow \tilde{w}_k / Z$ for all $k$ \Comment{Normalize}
\For {$t=1,2,\cdots,N'$}
\State Draw $k$ with weights $\{w_k\}$ and then draw $\theta_t \sim g_k(\theta)$
\EndFor
\State \Return $\{\theta_1, \theta_2, \cdots, \theta_{N'}\}$
\EndProcedure
\end{algorithmic}
\end{algorithm}

\subsection{Variance Reduction and Smoothing}

\paragraph{Random Tree Ensemble} Inspired by random forests \cite{breiman2001random,breiman1996bagging}, the full posterior is estimated by averaging $T$ independent trees output by Algorithm \ref{alg:tree}. Smoothing and averaging can reduce variance and yield better approximation accuracy. The trees can be built in parallel and resampling in Algorithm \ref{alg:onestage} only additionally requires picking a tree uniformly at random. 
\vspace{-1.2em}
\paragraph{Local Gaussian Smoothing} As another approach to increase smoothness, the blockwise uniform distribution in \eqref{eqs:full-est} can be replaced by a Gaussian distribution $g_k = N(\theta; \mu_k, \Sigma_k)$, with mean and covariance estimated ``locally'' by samples within the block. A multiplied Gaussian approximation is used: $\Sigma_k = (\sum_{i=1}^{m} \hat{\Sigma}_k^{(i)-1})^{-1}, \mu_k = \Sigma_k (\sum_{i=1}^m \hat{\Sigma}_k^{(i)-1} \hat{\mu}_k^{(i)})$, where $\hat{\Sigma}_k^{(i)}$ and $\hat{\mu}_k^{(i)}$ are estimated with the $i^{\text{th}}$ subset. We apply both random tree ensembles and local Gaussian smoothing in all applications of PART in this article unless explicitly stated otherwise.

\section{Theory}
In this section, we provide consistency theory (in the number of posterior samples) for histograms and the aggregated density.  We do not consider the variance reduction and smoothing modifications in these developments for simplicity in exposition, but extensions are possible.  Section 3.1 provides error bounds on ML and KD-tree partitioning-based histogram density estimators constructed from $N$ independent samples from a single joint posterior; modified bounds can be obtained for MCMC samples incorporating the mixing rate, but will not be considered here.  Section 3.2 then provides corresponding error bounds for our PART-aggregrated density estimators in the one-stage and pairwise cases. Detailed proofs are provided in the supplementary materials.

Let $f(\theta)$ be a $p$-dimensional posterior density function. Assume $f$ is supported on a measurable set $\Omega\subseteq \mathbb{R}^p$. Since one can always transform $\Omega$ to a bounded region by scaling, we simply assume $\Omega = [0, 1]^p$ as in \cite{liu2014multivariate,shen1994convergence} without loss of generality.  We also assume that $f\in C^1(\Omega)$.
\subsection{Space partitioning}
\paragraph{Maximum likelihood partition (ML)} For a given $K$, ML partition solves the following problem: 
\begin{align}
\hat f_{ML} &= \argmax \frac{1}{N}\sum_{k = 1}^K n_k\log\bigg(\frac{n_k}{N|A_k|}\bigg), \quad\mbox{s.t. } n_k/N \geq c_0\rho, ~|A_k|\geq \rho/D,\label{eq:prob3}
\end{align}
for some $c_0$ and $\rho$, where $D = \|f\|_\infty < \infty$. We have the following result.
\begin{theorem} \label{thm:1}
Choose $\rho = 1/K^{1+1/(2p)}$. For any $\delta > 0$, if the sample size satisfies that $N > 2(1-c_0)^{-2}K^{1+1/(2p)}\log(2K/\delta)$, then with probability at least $1 - \delta$, the optimal solution to \eqref{eq:prob3} satisfies that
\begin{align*}
D_{KL}(f\|\hat f_{ML}) \leq (C_1+2\log K)K^{-\frac{1}{2p}} + C_2\max\big\{\log D, 2\log K\big\}\sqrt{\frac{K}{N}\log\bigg(\frac{3e N}{K}\bigg)\log\bigg(\frac{8}{\delta}\bigg)},
\end{align*}
where $C_1 = \log D + 4pLD$ with $L = \|f'\|_\infty$ and $C_2 = 48\sqrt{p + 1}$.
\end{theorem}
When multiple densities $f^{(1)}(\theta),\cdots, f^{(m)}(\theta)$ are presented, our goal of imposing the same partition on all functions requires solving a different problem,
\begin{align}
(\hat f_{ML}^{(i)})_{i = 1}^m = \argmax \sum_{i=1}^m \frac{1}{N_i}\sum_{k = 1}^K n_k^{(i)}\log\bigg(\frac{n_k^{(i)}}{N_i|A_k|}\bigg), \quad\mbox{s.t. } n_k^{(i)}/N_i \geq c_0\rho, ~|A_k|\geq \rho/D,\label{eq:prob4}
\end{align}
where $N_i$ is the number of posterior samples for function $f^{(i)}$. A similar result as Theorem \ref{thm:1} for \eqref{eq:prob4} is provided in the supplementary materials.

\paragraph{Median partition/KD-tree (KD)} The KD-tree $\hat f_{KD}$ cuts at the empirical median for different dimensions. We have the following result.
\begin{theorem}\label{thm:2}
For any $\varepsilon > 0$, define $r_\varepsilon = \log_2\bigg(1 + \frac{1}{2 + 3L/\varepsilon}\bigg)$. For any $\delta > 0$, if $N > 32e^2(\log K)^2K\log (2K/\delta)$, then with probability at least $1 - \delta$, we have
\begin{align*}
\|\hat f_{KD} - f_{KD}\|_1\leq \varepsilon + pLK^{-{r_\varepsilon}/{p}} + 4e\log K\sqrt{\frac{2K}{N}\log\bigg(\frac{2K}{\delta}\bigg)}.
\end{align*}
If $f(\theta)$ is further lower bounded by some constant $b_0 > 0$, we can then obtain an upper bound on the KL-divergence. Define $r_{b_0} = \log_2\bigg(1 + \frac{1}{2 + 3L/b_0}\bigg)$ and we have
\begin{align*}
D_{KL}(f\|\hat f_{KD}) \leq \frac{pLD}{b_0}K^{-{r_{b_0}}/{p}} + 8e\log K\sqrt{\frac{2K}{N}\log\bigg(\frac{2K}{\delta}\bigg)}.
\end{align*}
\end{theorem}
When there are multiple functions and the median partition is performed on pooled data, the partition might not happen at the empirical median on each subset. However, as long as the partition quantiles are upper and lower bounded by $\alpha$ and $1 - \alpha$ for some $\alpha \in [1/2, 1)$, we can establish results similar to Theorem \ref{thm:2}. The result is provided in the supplementary materials.

\subsection{Posterior aggregation}
The previous section provides estimation error bounds on individual posterior densities, through which we can bound the distance between the true posterior conditional on the full data set and the aggregated density via \eqref{eqs:multiplcation}. Assume we have $m$ density functions $\{f^{(i)}, i = 1,2, \cdots, m\}$ and intend to approximate their aggregated density $f_I = \prod_{i \in I} f^{(i)} / \int \prod_{i\in I} f^{(i)}$, where $I = \{1, 2, \cdots, m\}$. Notice that for any $I'\subseteq I$, $f_{I'} = p(\theta | \bigcup_{i\in I'} X^{(i)})$. Let $D = \max_{I'\subseteq I} \|f_{I'}\|_\infty$, i.e., $D$ is an upper bound on all posterior densities formed by a subset of $X$. Also define $Z_{I'} = \int \prod_{i \in I'} f^{(i)}$. These quantities depend only on the model and the observed data (not posterior samples). We denote $\hat f_{ML}$ and $\hat f_{KD}$ by $\hat f$ as the following results apply similarly to both methods.
\par
The ``one-stage'' aggregation (Algorithm 2) first obtains an approximation for each $f^{(i)}$ (via either ML-partition or KD-partition) and then computes $\hat f_I = \prod_{i \in I} \hat f^{(i)} / \int \prod_{i\in I} \hat f^{(i)}$. 
\begin{theorem}[One-stage aggregation]\label{thm:one}
Denote the average total variation distance between $f^{(i)}$ and $\hat f^{(i)}$ by $\varepsilon$. Assume the conditions in Theorem \ref{thm:1} and \ref{thm:2} and for ML-partition
$$\sqrt N \geq 32c_0^{-1}\sqrt{2(p + 1)}K^{\frac{3}{2} + \frac{1}{2p}}\sqrt{\log\bigg(\frac{3e N}{K}\bigg)\log\bigg(\frac{8}{\delta}\bigg)}$$
and for KD-partition
$$N > 128e^2K(\log K)^2\log(K/\delta).$$ Then with high probability the total variation distance between $f_{I}$ and $\hat f_{I}$ is bounded by
\begin{align*}
 \|f_{I} - \hat f_{I}\|_1 \leq\frac{2}{Z_I} m(2D)^{m - 1}\varepsilon,
 \end{align*} 
where $Z_I$ is a constant that does not depend on the posterior samples.
\end{theorem}
\par
The approximation error of Algorithm 2 increases dramatically with the number of subsets. To ameliorate this, we introduce the {\em pairwise aggregation} strategy in Section 2, for which we have the following result.
\begin{theorem}[Pairwise aggregation]\label{thm:pairwise}
Denote the average total variation distance between $f^{(i)}$ and $\hat f^{(i)}$ by $\varepsilon$. Assume the conditions in Theorem \ref{thm:one}. Then with high probability the total variation distance between $f_{I}$ and $\hat f_{I}$ is bounded by
\begin{align*}
 \|f_{I} - \hat f_{I}\|_1 \leq (4C_0D)^{\log_2 m + 1}\varepsilon,
\end{align*}
where $C_0 = \max_{I''\subset I'\subseteq I} \frac{Z_{I''}Z_{I'\setminus I''}}{Z_{I'}}$ is a constant that does not depend on posterior samples.
\end{theorem}

\section{Experiments}

In this section, we evaluate the empirical performance of {\em PART} and compare the two algorithms {\em PART-KD} and {\em PART-ML} to the following posterior aggregation algorithms. 

\begin{enumerate}
\item \textbf{Simple averaging} (\emph{average}): each aggregated sample is an arithmetic average of $M$ samples coming from $M$ subsets. 

\item \textbf{Weighted averaging} (\emph{weighted}): also called \textbf{Consensus Monte Carlo} algorithm \cite{scott2013bayes}, where each aggregated sample is a weighted average of $M$ samples. The weights are optimally chosen for a Gaussian posterior. 

\item \textbf{Weierstrass rejection sampler} (\emph{Weierstrass}): subset posterior samples are passed through a rejection sampler based on the Weierstrass transform to produce the aggregated samples~\cite{wang2013parallel}. We use its R package\footnote{\url{https://github.com/wwrechard/weierstrass}} for experiments.

\item \textbf{Parametric density product} (\emph{parametric}): aggregated samples are drawn from a multivariate Gaussian, which is a product of Laplacian approximations to subset posteriors~\cite{neiswanger2013asymptotically}. 

\item \textbf{Nonparametric density product} (\emph{nonparametric}): aggregated posterior is approximated by a product of kernel density estimates of subset posteriors~\cite{neiswanger2013asymptotically}. Samples are drawn with an independent Metropolis sampler. 

\item \textbf{Semiparametric density product} (\emph{semiparametric}): similar to the \emph{nonparametric}, but with subset posteriors estimated semiparametrically~\cite{neiswanger2013asymptotically,hjort1995nonparametric}.

\end{enumerate}

All experiments except the two toy examples use adaptive MCMC \cite{haario2006dram,haario2001adaptive} \footnote{\url{http://helios.fmi.fi/~lainema/mcmc/}} for posterior sampling. For {\em PART-KD/ML}, one-stage aggregation (Algorithm \ref{alg:onestage}) is used only for the toy examples (results from pairwise aggregation are provided in the supplement). For other experiments,  pairwise aggregation is used, which draws 50,000 samples for intermediate stages and halves $\delta_{\rho}$ after each stage to refine the resolution (The value of $\delta_\rho$ listed below is for the final stage). The random ensemble of {\em PART} consists of 40 trees.

\subsection{Two Toy Examples}
The two toy examples highlight the performance of our methods in terms of (i) recovering multiple modes and (ii) correctly locating posterior mass when subset posteriors are heterogeneous. The {\em PART-KD/PART-ML} results are obtained from Algorithm \ref{alg:onestage} without local Gaussian smoothing. 

\paragraph{Bimodal Example}
Figure \ref{fig:example-bimodal} shows an example consisting of $m=10$ subsets. Each subset consists of 10,000 samples drawn from a mixture of two univariate normals $0.27 N(\mu_{i,1}, \sigma_{i,1}^2) + 0.73 N(\mu_{i,2}, \sigma_{i,2}^2)$, with the means and standard deviations slightly different across subsets, given by $\mu_{i,1} = -5 + \epsilon_{i,1}$, $ \mu_{i,2}=5+\epsilon_{i,2} $ and $\sigma_{i,1} = 1 + |\delta_{i,1}|$, $\sigma_{i,2} = 4+|\delta_{i,2}|$, where $\epsilon_{i,l} \sim N(0,0.5)$, $\delta_{i,l} \sim N(0, 0.1)$ independently for $m=1,\cdots,10$ and $l=1,2$. The resulting true combined posterior (red solid) consists of two modes with different scales. In Figure \ref{fig:example-bimodal}, the left panel shows the subset posteriors (dashed) and the true posterior; the right panel compares the results with various methods to the truth. A few are omitted in the graph: \emph{average} and \emph{weighted average} overlap with \emph{parametric}, and \emph{Weierstrass} overlaps with \emph{PART-KD/PART-ML}. 

\begin{figure}[!htb]
\centering
\includegraphics[width=0.44\textwidth]{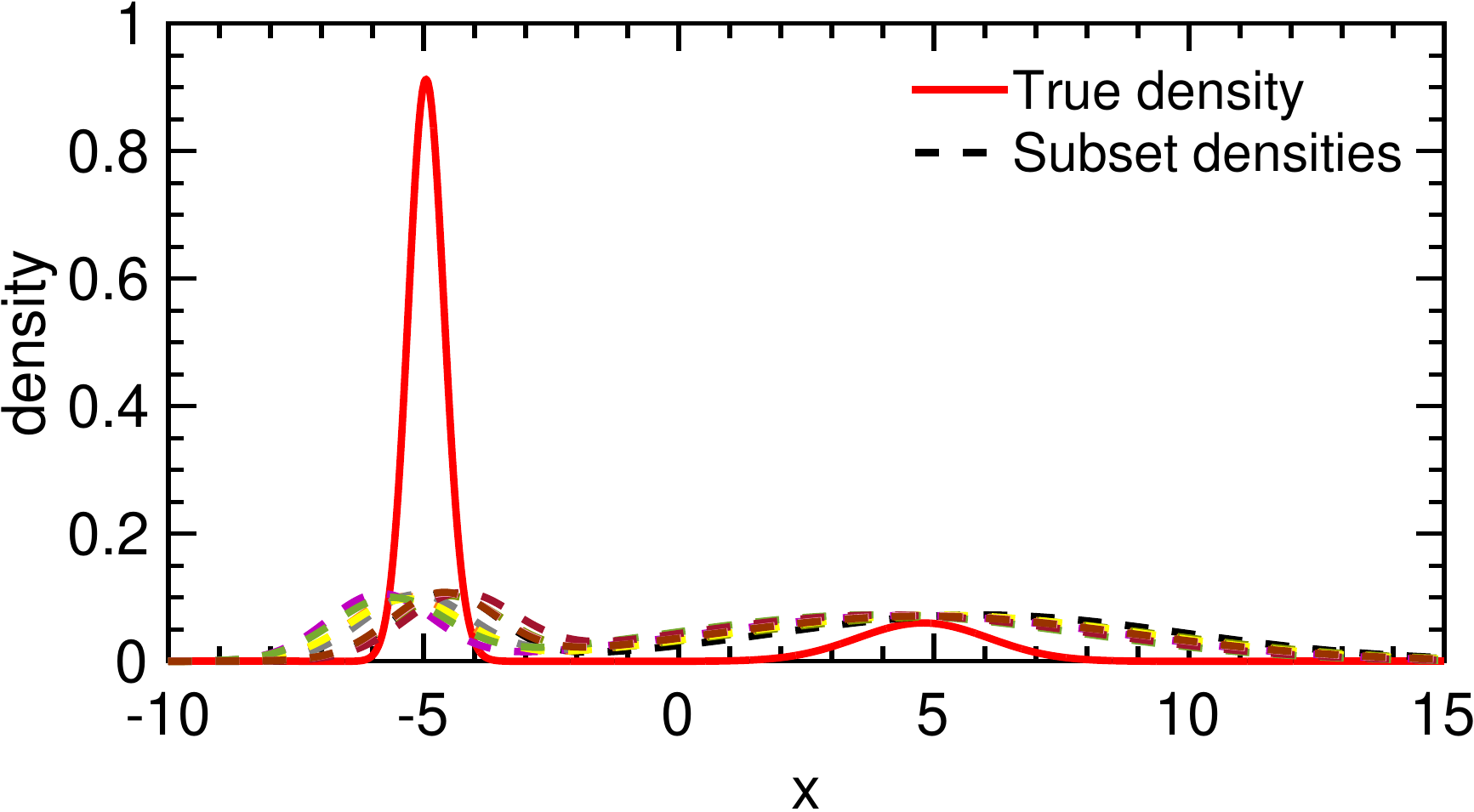}
\includegraphics[width=0.4\textwidth]{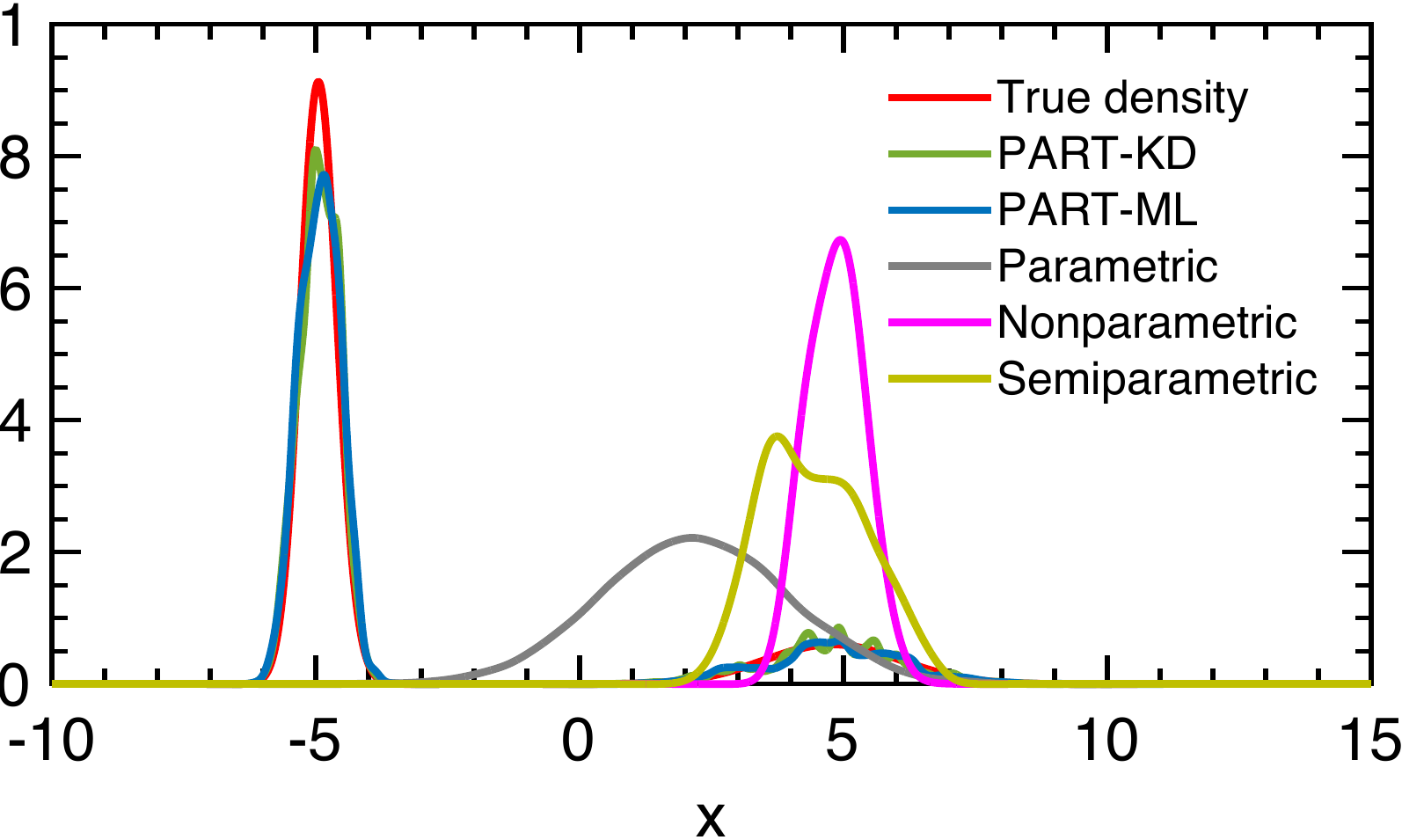}
\caption{Bimodal posterior combined from 10 subsets. Left: the true posterior and subset posteriors (dashed). Right: aggregated posterior output by various methods compared to the truth. Results are based on 10,000 aggregated samples.}
\label{fig:example-bimodal}
\end{figure}

\vspace{-1.5em}
\paragraph{Rare Bernoulli Example} We consider $N=10,000$ Bernoulli trials $x_i \overset{\text{iid}}{\sim} \text{Ber}(\theta)$ split into $m=15$ subsets. The parameter $\theta$ is chosen to be $2m/N$ so that on average each subset only contains 2 successes.  By random partitioning, the subset posteriors are rather heterogeneous as plotted in dashed lines in the left panel of Figure \ref{fig:example-rare}. The prior is set as $\pi(\theta) = \text{Beta}(\theta;2,2)$. The right panel of Figure \ref{fig:example-rare} compares the results of various methods. {\em PART-KD}, {\em PART-ML} and \emph{Weierstrass} capture the true posterior shape, while \emph{parametric}, \emph{average} and \emph{weighted average} are all biased. The \emph{nonparametric} and \emph{semiparametric} methods produce flat densities near zero (not visible in Figure \ref{fig:example-rare} due to the scale). 

\begin{figure}[!htb]
\centering
\includegraphics[width=0.45\textwidth]{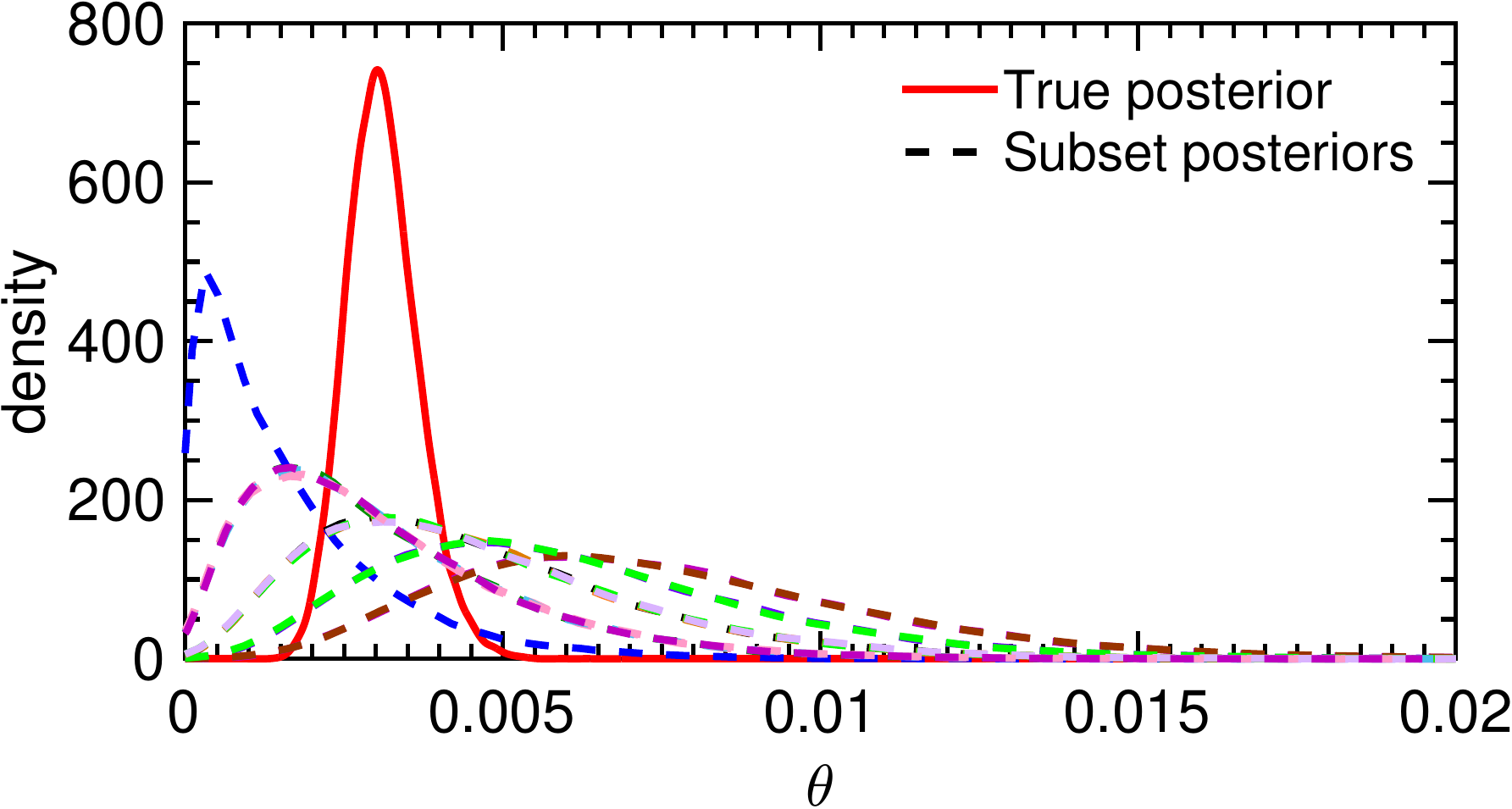}
\includegraphics[width=0.44\textwidth]{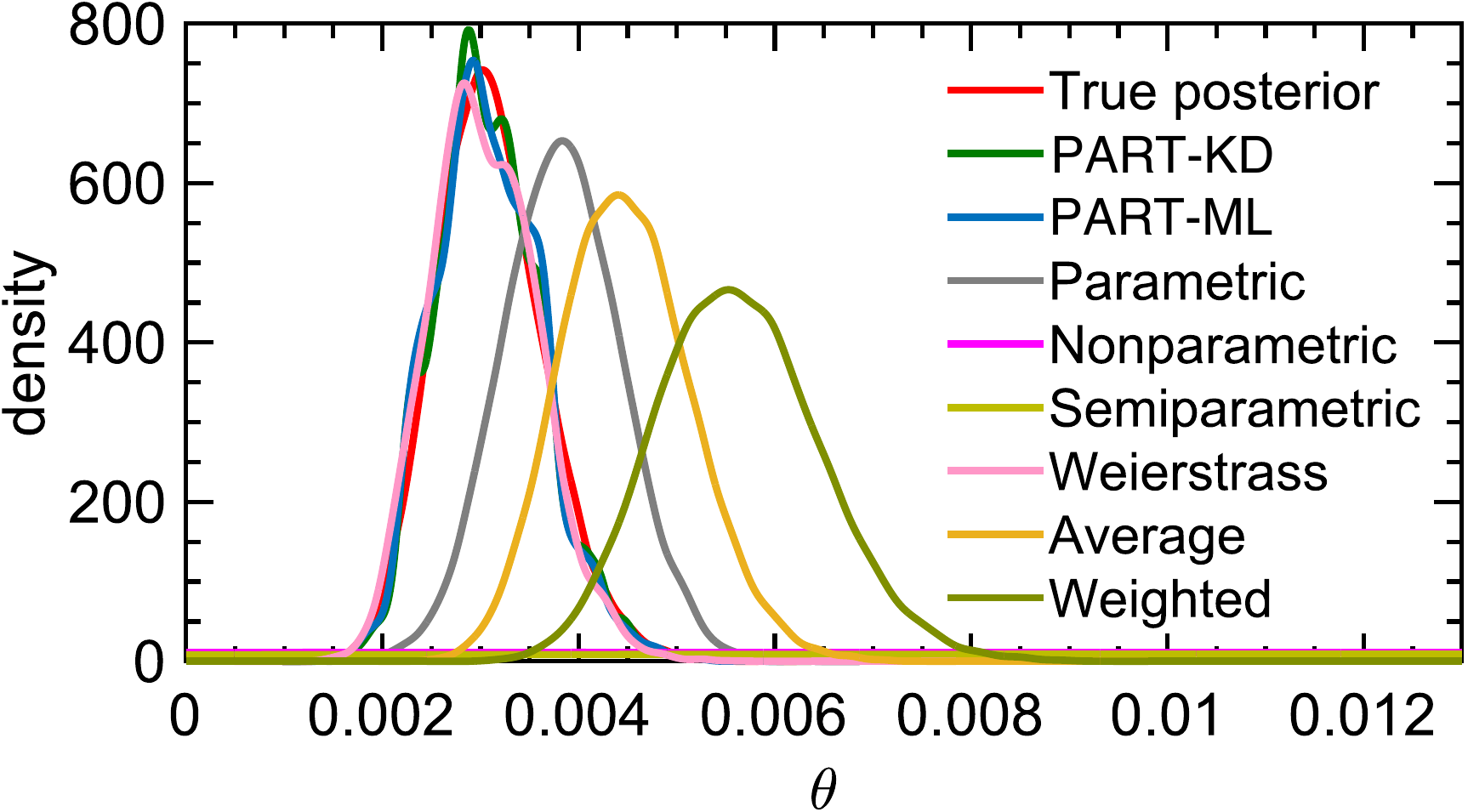}
\caption{The posterior for the probability $\theta$ of a rare event. Left: the full posterior (solid) and $m=15$ subset posteriors (dashed). Right: aggregated posterior output by various methods. All results are based on 20,000 aggregated samples.}
\label{fig:example-rare}
\end{figure}

\subsection{Bayesian Logistic Regression}
\paragraph{Synthetic dataset} The dataset $\{(\mathbf{x}_i, y_i)\}_{i=1}^{N}$ consists of $N=50,000$ observations in $p=50$ dimensions. All features $\mathbf{x}_i \in \mathbb{R}^{p-1}$ are drawn from $N_{p-1}(\mathbf{0}, \mathbf{\Sigma})$ with $p=50$ and $\Sigma_{k,l} = 0.9^{|k-l|}$. The model intercept is set to $-3$ and the other coefficient $\bm{\theta}^{\ast}_j$'s are drawn randomly from $N(0, 5^2)$. Conditional on $\mathbf{x}_i$, $y_i \in \{0,1\}$ follows $p(y_i=1) = 1/(1+\exp(-\bm{\theta}^{\ast T} [1, \mathbf{x}_i]))$. The dataset is randomly split into $m=40$ subsets. For both full chain and subset chains, we run adaptive MCMC for 200,000 iterations after 100,000 burn-in. Thinning by 4 results in 
$T=50,000$ samples. 

The samples from the full chain (denoted as $\{\bm{\theta}_j\}_{j=1}^{T}$) are treated as the ground truth. To compare the accuracy of different methods, we resample $T$ points $\{\hat{\bm{\theta}}_j\}$ from each aggregated posterior and then compare them using the following metrics: (1) RMSE of posterior mean $\| \frac{1}{pT}(\sum_{j} \hat{\bm{\theta}}_j - \sum_{j} \bm{\theta}_j)\|_2 $ (2) approximate KL divergence $ D_{\text{KL}}({p}(\bm{\theta}) \| \hat{p}(\bm{\theta})) $ and $ D_{\text{KL}}(\hat{p}(\bm{\theta}) \| p(\bm{\theta})) $, where $\hat{p}$ and $p$ are both approximated by multivariate Gaussians (3) the posterior concentration ratio, defined as $r = \sqrt{{\sum_j \| \hat{\bm{\theta}}_j - \bm{\theta}^{\ast}\|_2^2}/{\sum_j \| \bm{\theta}_j - \bm{\theta}^{\ast}\|_2^2}}$, which measures how posterior spreads out around the true value (with $r=1$ being ideal). The result is provided in Table~\ref{tab:logistic}. Figure \ref{fig:logisitic-vslength} shows the $D_{\text{KL}}(p \| \hat p)$ versus the length of subset chains supplied to the aggregation algorithm. The results of \emph{PART} are obtained with $\delta_{\rho}=0.001$, $\delta_{a}=0.0001$ and 40 trees. Figure \ref{fig:logistic-dist} showcases the aggregated posterior for two parameters in terms of joint and marginal distributions.

\begin{figure}[!htb]
\CenterFloatBoxes
\begin{floatrow}
\capbtabbox{%
\begin{tabular}{@{}lcccc@{}}
\toprule
Method           & RMSE & $D_\text{KL} (p \| \hat{p})$ & $D_\text{KL} (\hat{p} \| p)$ & $r$ \\ \midrule
PART (KD)        & \textbf{0.587}         & ${3.95 \times 10^2}$  & ${6.45 \times 10^2}$  & \textbf{3.94}   \\
PART (ML)        & 1.399                  & $\mathbf{8.05 \times 10^1}$               & $\mathbf{5.47 \times 10^2}$              & 9.17            \\
average          & 29.93                  & $2.53 \times 10^3$           & $5.41 \times 10^4$           & 184.62          \\
weighted  & 38.28                  & $2.60 \times 10^4$           & $2.53 \times 10^5$           & 236.15          \\
Weierstrass      & 6.47                   & $7.20 \times 10^2$           & $2.62 \times 10^3$           & 39.96           \\
parametric       & 10.07                  & $2.46 \times 10^3$           & $6.12 \times 10^3$           & 62.13           \\
nonparametric    & 25.59                  & $3.40 \times 10^4$           & $3.95 \times 10^4$           & 157.86          \\
semiparametric   & 25.45                  & $2.06 \times 10^4$           & $3.90 \times 10^4$           & 156.97          \\
\bottomrule
\end{tabular}
}{
\caption{Accuracy of posterior aggregation on logistic regression.}\label{tab:logistic}
}
\hspace{-4.2em}
\ffigbox{%
\includegraphics[width=0.35\textwidth]{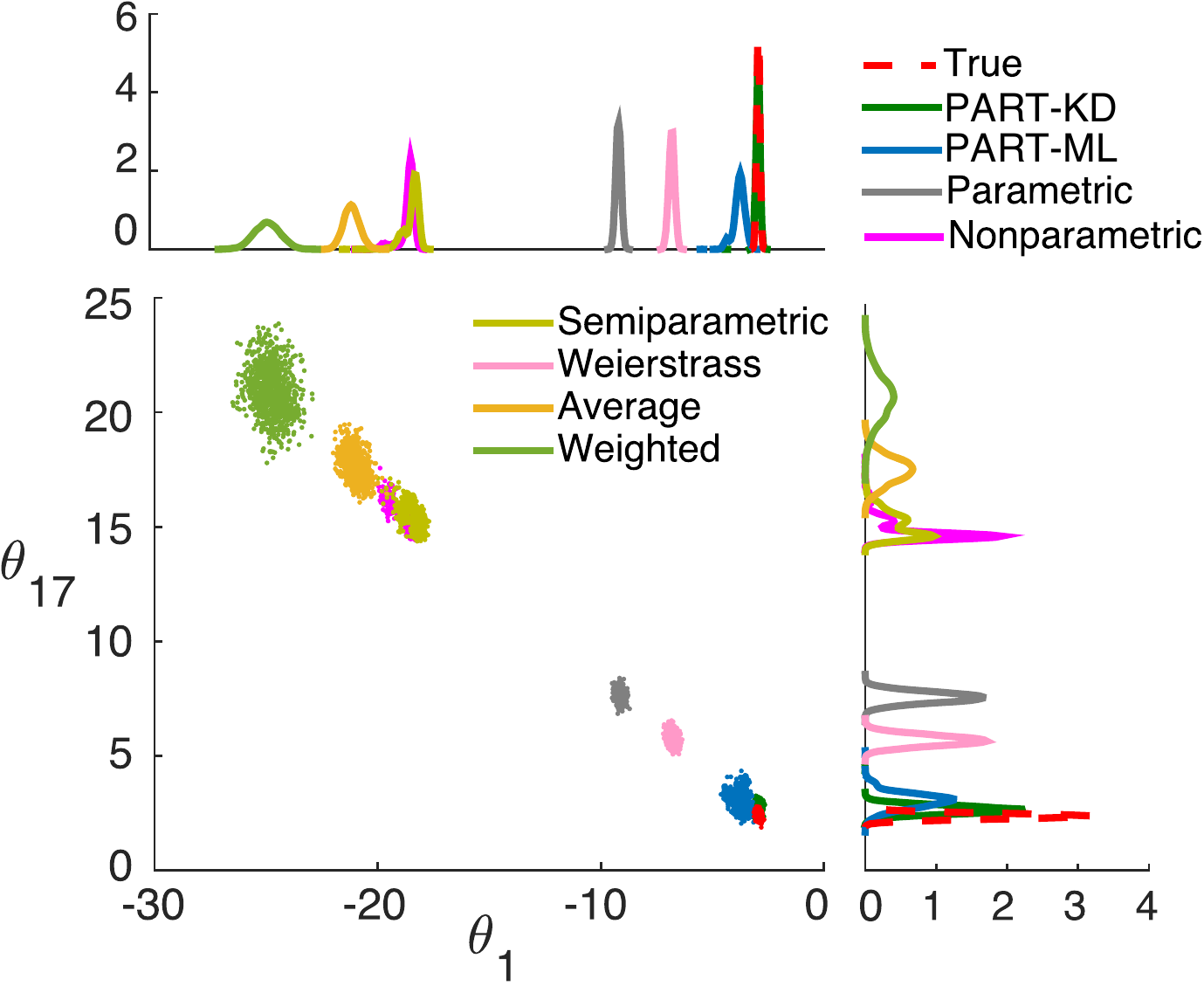}
}{%
\caption{Posterior of $\theta_{1}$ and $\theta_{17}$.}\label{fig:logistic-dist}%
}
\end{floatrow}
\end{figure}

\vspace{-1em}
\paragraph{Real datasets} We also run experiments on two real datasets: (1) the \emph{Covertype} dataset\footnote{\url{http://www.csie.ntu.edu.tw/~cjlin/libsvmtools/datasets/binary.html}}~\cite{blackard1998comparative} consists of 581,012 observations in 54 dimensions, and the task is to predict the type of forest cover with cartographic measurements; (2) the \emph{MiniBooNE} dataset\footnote{\url{https://archive.ics.uci.edu/ml/machine-learning-databases/00199}}~\cite{roe2005boosted,Lichman2013} consists of 130,065 observations in 50 dimensions, whose task is to distinguish electron neutrinos from muon neutrinos with experimental data. For both datasets, we reserve $1/5$ of the data as the test set. The training set is randomly split into $m=50$ and $m=25$ subsets respectively for \emph{covertype} and \emph{MiniBooNE}. Figure~\ref{fig:realdata} shows the prediction accuracy versus total runtime (parallel subset MCMC + aggregation time) for different methods. For each MCMC chain, the first 20\% iterations are discarded before aggregation as burn-in. The aggregated chain is required to be of the same length as the subset chains. As a reference, we also plot the result for the full chain and {\em lasso} \cite{tibshirani1996regression} run on the full training set. 

\begin{figure}[!htb]
    \begin{floatrow}
     \ffigbox[\FBwidth]
       {\includegraphics[width=0.33\textwidth]{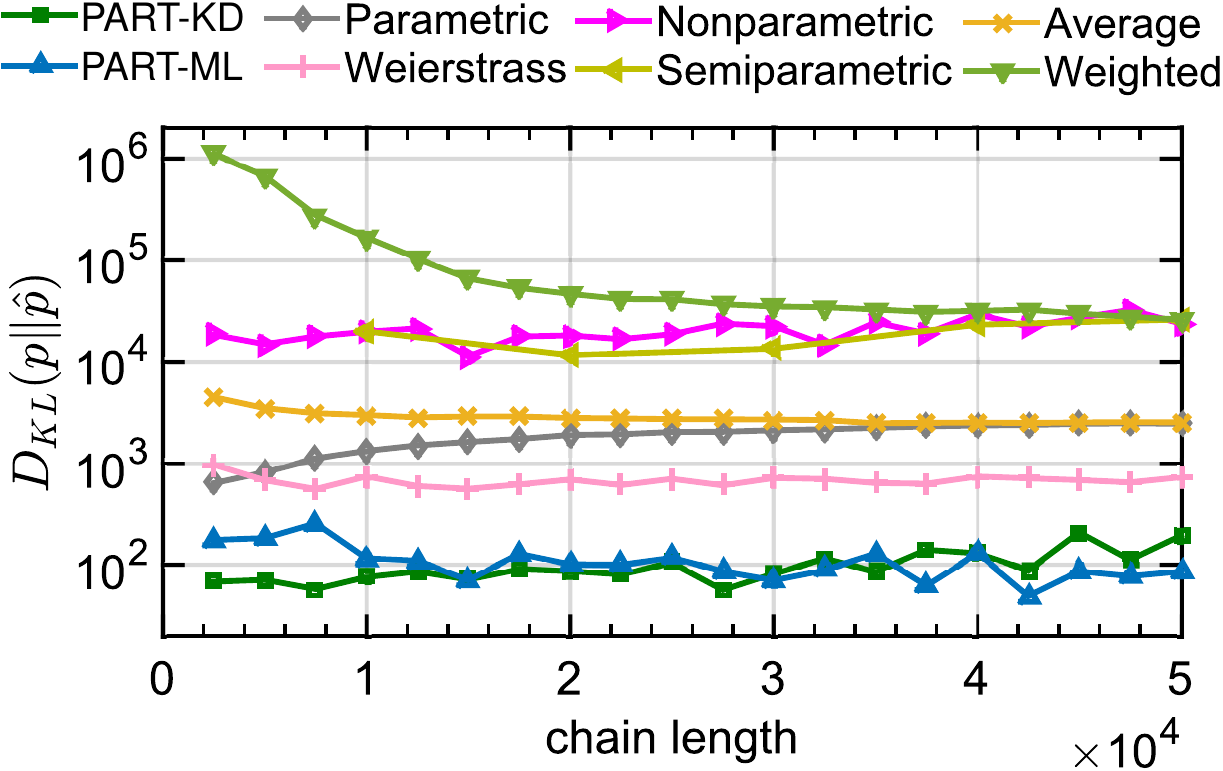}}{%
\caption{Approximate KL divergence between the full chain and the combined posterior versus the length of subset chains.}
\label{fig:logisitic-vslength} 
}
     \ffigbox[\Xhsize]
       {\includegraphics[width=0.64\textwidth]{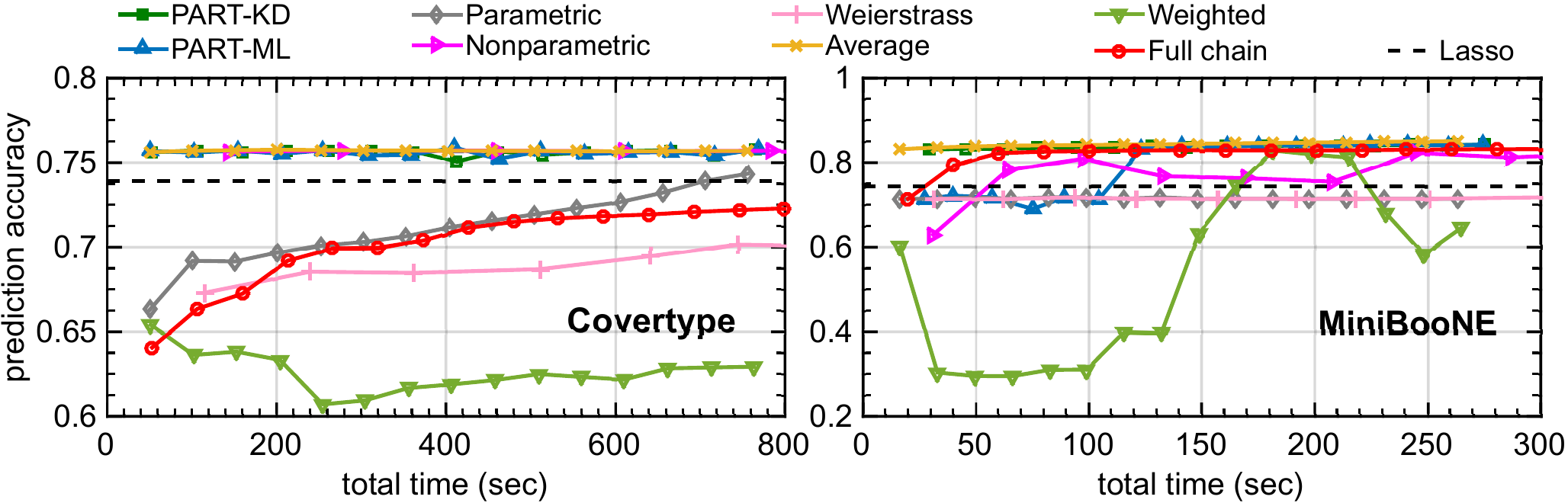}}{%
\caption{Prediction accuracy versus total runtime (running chain + aggregation) on \emph{Covertype} and \emph{MiniBooNE} datasets (\emph{semiparametric} is not compared due to its long running time). Plots against the length of chain are provided in the supplement.}
\label{fig:realdata}
}
    \end{floatrow}
\end{figure}

\vspace{-1em}
\section{Conclusion}
In this article, we propose a new embarrassingly-parallel MCMC algorithm {\em PART} that can efficiently draw posterior samples for large data sets. {\em PART} is simple to implement, efficient in subset combining and has theoretical guarantees. Compared to existing EP-MCMC algorithms, {\em PART} has substantially improved performance. 
Possible future directions include (1) exploring other multi-scale density estimators which share similar properties as partition trees but with a better approximation accuracy (2) developing a tuning procedure for choosing good $\delta_{\rho}$ and $\delta_{a}$, which are essential to the performance of \emph{PART}.

\bibliographystyle{unsrt}
\bibliography{tree15}

\include{tree15-supplement}

\end{document}

%% file: tree15-supplement.tex
\clearpage
\setcounter{section}{0}
\setcounter{footnote}{0}

{
\hsize\textwidth
\linewidth\hsize \vskip 0.1in \toptitlebar \centering
{\LARGE\bf Supplementary Materials \par}  \bottomtitlebar 
}

\section*{Appendix A: Schematic illustration of the algorithm}
Here we include a schematic illustration of the density aggregation in \text{PART} algorithm.
\begin{figure}[!h]
\centering
\includegraphics[width=0.8\textwidth]{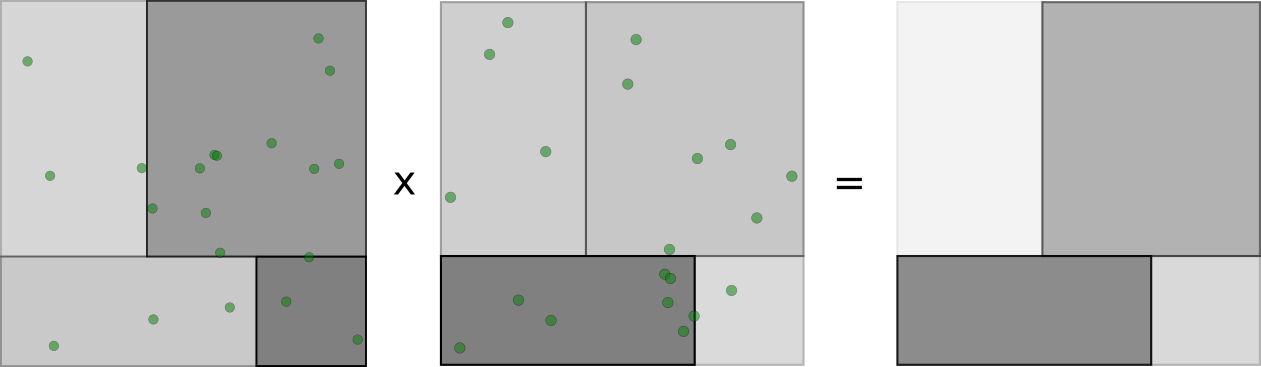}
\caption{A schematic figure illustrating the \textit{density aggregation} step of the algorithm. Two trees in the left share the same block structure and the aggregated histogram is obtained by block-wise multiplication and renormalization.}
\label{fig:schematic}
\end{figure}

\section*{Appendix B: Proof of Theorem \ref{thm:1}}
Let $\Gamma_{K,\rho}$ be a subset of $\Gamma_K$ defined as $\Gamma_{K,\rho} = \big\{f_0 \in \Gamma_K| ~\min_{A_k} \mathbbm{E}\tmmathbf{1}_{A_k} \geq \rho \big\} $. We prove a more general form of Theorem \ref{thm:1} here.
\begin{theorem}
For any $\delta > 0$, if the sample size satisfies that $N > \max\{K, ~\frac{2\log(2K/\delta)}{\rho(1-c_0)^2}\}$, then with probability at least $1 - \delta$, the optimal solution to \eqref{eq:prob3} satisfies that
\begin{align*}
D_{KL}(f\|\hat f_{ML}) \leq \min_{f_0\in \Gamma_{K,\rho}} D_{KL}(f\|f_0) + C_\rho\sqrt{\frac{K}{N}\log\bigg(\frac{3e N}{K}\bigg)\log\bigg(\frac{8}{\delta}\bigg)},
\end{align*}
where $C_\rho = 48\sqrt{p + 1}\max\big\{\log D, \log \rho^{-1}\big\}$. 
\par
Now choosing $\rho = 1/K^{1+1/(2p)}$, the condition becomes $N > 2(1-c_0)^{-2}K^{1+1/(2p)}\log(2K/\delta)$, then with probability at least $1 - \delta$ we have
\begin{align*}
D_{KL}(f\|\hat f_{ML}) \leq (C_1+2\log K)K^{-\frac{1}{2p}} + C_2\max\big\{\log D, 2\log K\big\}\sqrt{\frac{K}{N}\log\bigg(\frac{3e N}{K}\bigg)\log\bigg(\frac{8}{\delta}\bigg)},
\end{align*}
where $C_1 = 2\log D + 4pLD$ with $L = \|f'\|_\infty$ and $C_2 = 48\sqrt{p + 1}$.
\end{theorem}
When multiple densities $f^{(1)}(\theta),\cdots, f^{(m)}(\theta)$ are presented, our goal of imposing the same partition on all functions requires solving a different problem \eqref{eq:prob4}.
As long as $\mathbbm{\hat E}\sum_{i = 1}^m \hat f_{ML}^{(i)} \geq \mathbbm{\hat E}\sum_{i = 1}^m f_{K, \rho}^{(i)}$ remains true, where $f_{K, \rho}^{(i)} = \argmin_{f_0\in \Gamma_{K, \rho}} D_{KL}(f^{(i)}\|f_0)$, the whole proof of Theorem \ref{thm:1} is also valid for \eqref{eq:prob4}. Therefore, we have the following Corollary.
\begin{corollary}[$m$ copies]\label{cor:1}
For any $\delta > 0$, if the sample size satisfies that $N > 2(1-c_0)^{-2}K^{1+1/(2p)}\log(2mK/\delta)$ and $\rho = 1/K^{1+1/(2p)}$, then with probability at least $1 - \delta$, the optimal solution to \eqref{eq:prob4} satisfies that
\begin{align*}
\frac{1}{m}\sum_{i = 1}^m D_{KL}(f^{(i)}\|\hat f_{ML}^{(i)}) &\leq (C_1+2\log K)K^{-\frac{1}{2p}} + C_2\max\big\{\log D, 2\log K\big\}\sqrt{\frac{K}{N}\log\bigg(\frac{3e N}{K}\bigg)\log\bigg(\frac{8m}{\delta}\bigg)}.
\end{align*}
\end{corollary}

To prove Theorem \ref{thm:1} we need the following lemmas.
\begin{lemma}\label{lemma:opt}
The optimal solution of \eqref{eq:prob3} is also the optimal solution of the following problem,
\begin{align}
\hat f_{ML} &= \argmax_{f\in \Gamma_K} \frac{1}{N}\sum_{k = 1}^K n_k\log \hat\pi_k, \quad\mbox{s.t. } n_k\geq c_0\rho N, ~|A_k|\geq \rho/D \mbox{ and } \sum_{k = 1}^K \hat \pi_k|A_k| = 1. \label{eq:prob2}
\end{align}
\end{lemma}

\begin{proof}
We write out the empirical log likelihood
\begin{align*}
\frac{1}{N}\sum_{k = 1}^K n_k\log \hat\pi_k = \sum_{k = 1}^K \frac{n_k}{N}\log \hat\pi_k|A_k| - \sum_{n_k \geq 0} \frac{n_k}{N}\log |A_k|.
\end{align*}
For any fixed partition $\{A_k\}$, under the constraint that $\sum_{k = 1}^K \hat\pi_k|A_k| = 1$, one can easily see that
\begin{align*}
\hat\pi_k|A_k| = \frac{n_k}{N}
\end{align*}
maximizes the result.
\end{proof}

Next, we show that the optimal approximation 
\begin{align*}
f_{K,\rho} = \argmin_{\hat f_{ML}\in \Gamma_{K, \rho}} D_{KL}(f\|\hat f_{ML})
\end{align*}
is a feasible solution to \eqref{eq:prob2} with a high probability. 

\begin{lemma}\label{lemma:domain}
 Let $f_{K, \rho}$ be the optimal approximation in $\Gamma_{K,\rho}$, then $f_{K, \rho}$ satisfies that $\min_k |A_k| \geq \rho/D$. In addition, with probability at least $1 - K\exp(- (1-c_0)^2\rho N/2)$, we have $n_k/N \geq c_0\rho$, i.e., $f_{K, \rho}$ is a feasible soluton of \eqref{eq:prob2}.
 \end{lemma} 
\begin{proof}
Let $n_k$ be the counts of data points on the partition of $f_{K, \rho}$. Notice $f_{K,\rho}$ is a fixed function that does not depend on the data. Therefore, each $n_k$ follows a binomial distribution. Define $P(A_k) = \mathbbm{E}\tmmathbf{1}_{A_k}$. According to the definition of $\Gamma_{K,\rho}$, we have $P(A_k) \geq \rho$. Using the Chernoff's inequality, we have for any $0 < \delta < 1$,
\begin{align*}
P\bigg(\frac{n_k}{N} \leq (1 - \delta)P(A_k)\bigg) \leq \exp\bigg(-\frac{\delta^2NP(A_k)}{2}\bigg).
\end{align*}
Taking $\delta = 1-c_0$ and union bounds we have
\begin{align*}
P\bigg(\min_{k} \frac{n_k}{N} \geq c_0\rho\bigg) \geq 1 - K\exp(- (1-c_0)^2\rho N/2).
\end{align*}
On the other hand, the following inequality shows the bound on $|A_k|$,
\begin{align*}
|A_k| = \int_{A_k} 1 \geq \int_{A_k} f/D \geq \rho/D.
\end{align*}
\end{proof}

Lemma \ref{lemma:domain} states that with a high probability we have $\mathbbm{\hat E}\log \hat f_{ML} \geq \mathbbm{\hat E}\log f_{K,\rho}$. This result will be used to prove our main theorem. 
\par
Although the actual partition algorithm selects the dimension for partitioning completely at random for each iteration, in the proof we will assume one predetermined order of partition (such as $\{1, 2, 3, \cdots, p, 1, 2, \cdots\}$) just for simplicity. The order of partitioning does not matter as long as every dimension receives sufficient number of partitions. When the selection is randomly taken, with high probability (increasing exponentially with $N$), the number of partitions in each dimension will concentrate around the average. Thus, it suffices to prove the result for the simple $\{1, 2, 3, \cdots, p, 1, 2, \cdots\}$ case.

\begin{proof}[\textbf{Proof of Theorem \ref{thm:1}}]
The proof consists of two parts, namely (1) bounding the excess loss compared to the optimal approximation $ f_{K,\rho} $ in $ \Gamma_{K,\rho} $ and (2) bounding the error between the optimal approximation and the true density. 

For the first part, using the fact that $\mathbbm{\hat E}\log f_{K,\rho}(\theta) \leq \mathbbm{\hat E}\log \hat f_{ML}(\theta)$, the excess loss can be expressed as
\begin{align*}
D_{KL} (f\|\hat{f}_{ML}) &- D_{KL}(f\|f_{K,\rho}) = \mathbbm{E}\log f_{K,\rho}(\theta) - \mathbbm{E}\log \hat f_{ML}(\theta)\\
&= \mathbbm{E}\log f_{K,\rho}(\theta) - \mathbbm{\hat E}\log f_{K,\rho}(\theta) + \mathbbm{\hat E}\log f_{K,\rho}(\theta) \\
&\quad - \mathbbm{\hat E}\log \hat f_{ML}(\theta) + \mathbbm{\hat E}\log \hat f_{ML}(\theta) - \mathbbm{E}\log \hat f_{ML}(\theta)\\
&\leq \mathbbm{E}\log f_{K,\rho}(\theta) - \mathbbm{\hat E}\log f_{K,\rho}(\theta)+ \mathbbm{\hat E}\log \hat f_{ML}(\theta) - \mathbbm{E}\log \hat f_{ML}(\theta).
\end{align*}
Assuming the partitions for $f_{K,\rho}$ and $\hat f_{ML}$ are $\{A_k\}$ and $\{\hat A_k\}$ respectively, we have
\begin{align}
\notag D_{KL} (f\|\hat{f}_{ML}) &- D_{KL}(f\|f_{K,\rho}) = \sum_{k = 1}^K \log \pi_k\big(\mathbbm{E}\tmmathbf{1}_{A_k} - \mathbbm{\hat E}\tmmathbf{1}_{A_k}\big) + \sum_{k = 1}^K \log \frac{n_k}{N|\hat A_k|} \big(\mathbbm{E}\tmmathbf{1}_{\hat A_k} - \mathbbm{\hat E}\tmmathbf{1}_{\hat A_k}\big)\\
&\leq \bigg(\max_{k}|\log \pi_k| + \max_k \big|\log \frac{n_k}{N|\hat A_k|}\big|\bigg) \sup_{\{A_k\}\in \mathcal{F}_k}\sum_{k = 1}^K |\mathbbm{E}\tmmathbf{1}_{A_k} - \mathbbm{\hat E}\tmmathbf{1}_{A_k}|. \label{eq:KL}
\end{align}
Following a similar argument as Lemma \ref{lemma:opt}, for $f_{K,\rho}$ we can prove that $\pi_k = \int_{A_k} f(\theta)d\theta/|A_k|$, thus we have $\rho_0 \leq \pi_k \leq D$ for any $1 \leq k \leq K$. Similarly for $\frac{n_k}{N|A_k|}$ we have $\rho\leq \frac{n_k}{N|A_k|} \leq D/\rho$. Therefore, the first term in \eqref{eq:KL} can be bounded as
\begin{align*}
\max_{k}|\log \pi_k| + \max_k \big|\log \frac{n_k}{N|\hat A_k|}\big|\leq 3\max\{\log D, ~\log \rho^{-1}\}.
\end{align*}
The second term in \eqref{eq:KL} is the concentration of the empirical measure over all possible K-rectangular partitions. Using the result from \cite{chen1987almost}, we have the following large deviation inequality. For any $\epsilon \in (0, 1)$, we have
\begin{align}
P\bigg(\sup_{\{A_k\}\in \mathcal{F}_k}\sum_{k = 1}^K |\mathbbm{E}\tmmathbf{1}_{A_k} - \mathbbm{\hat E}\tmmathbf{1}_{A_k}| > \epsilon \bigg) < 4\exp\bigg\{-\frac{\epsilon^2 N}{2^9}\bigg\},\label{eq:concentration}
\end{align}
if $N \geq \max\{K, (100\log 6)/\epsilon^2, 2^9(p + 1)K\log(3e N/K)/\epsilon^2\}$. For any $\delta > 0$, taking $\epsilon = 2^9(p + 1)K\log(3e N/K)/N \log(4/\delta)$, we have that
\begin{align*}
\sup_{\{A_k\}\in \mathcal{F}_k}\sum_{k = 1}^K |\mathbbm{E}\tmmathbf{1}_{A_k} - \mathbbm{\hat E}\tmmathbf{1}_{A_k}|\leq 16\sqrt{2(p + 1)}\sqrt{\frac{K}{N}\log\bigg(\frac{3e N}{K}\bigg)\log\bigg(\frac{8}{\delta}\bigg)},
\end{align*}
with probability at least $1 - \delta/2$. Define $C_\rho = 48\sqrt{2(p + 1)}\max\{\log D, \log \rho^{-1}\}$. When $N > \frac{2\log(2K/\delta)}{\rho(1-c_0)^2}$, Lemma \ref{lemma:domain} holds with probability at least $1 - \delta/2$. Taking the union bound, we have
\begin{align}
D_{KL}(f\|\hat f_{ML}) \leq \min_{f_0\in \Gamma_{K,\rho}} D_{KL}(f\|f_0) + C_\rho\sqrt{\frac{K}{N}\log\bigg(\frac{3e N}{K}\bigg)\log\bigg(\frac{8}{\delta}\bigg)}\label{eq:thm1.first}
\end{align}
holds with probability greater than $1 - \delta$.
\par
To prove the second part, we construct one reference density $\tilde f \in \Gamma_{K, \rho}$ that gives the error specified in the theorem. According to the argument provided in the paragraph prior to this proof, we assume the dimension that we cut at each iteration follows an order $\{1, 2, \cdots, p, 1, 2, \cdots\}$. We then construct $f_0$ in the following way. At iteration $i$, we check the probability on the whole region. 
\begin{itemize}
\item[i] If the probability is greater than $2\rho$, we then cut at the midpoint of the selected dimension. If the resulting two blocks $B_1$ and $B_2$ satisfy that $P_f(B_1) \geq \rho$ and $P_f(B_2) \geq \rho$, we continue to the next iteration. However, if any of them fails to satisfy the condition, we find the minimum-deviated cut that satisfies the probability requirement. 
\item[ii] If the probability on the whole region is less than $2\rho$, we stop cutting on this region and move to the next region for the current iteration. 
\end{itemize}
It is easy to show that as long as $\rho\leq 1/(2K)$, the above procedure is able to yield a $K$-block partition $\{\tilde A_k\}$ before termination. Finally, the reference density $\tilde f$ is defined as
\begin{align}
\tilde f(\theta) = \sum_{k = 1}^K \frac{\int_{\tilde A_k} f(\theta)d\theta}{|\tilde A_k|} \tmop{1}_{\tilde A_k}(\theta).
\end{align}
\par
The construction procedure ensures the following property for $\tilde f\in \Gamma_{K, \rho}$. Assuming $K \in [2^d, 2^{d+1})$ for some $d > 0$, then each $\tilde A_k$ must fall into either of the following two categories (could be both),
\begin{enumerate}
\item $\rho \leq P_f(\tilde A_k) \leq 2\rho$,
\item $P_f(\tilde A_k)\geq \rho$ and the longest edge of cube $\tilde A_k$ must be less than $2^{-\lfloor d/p\rfloor}\leq 2^{-d/p + 1}\leq 4k^{-1/p}$.
\end{enumerate}
We use $I_1$ and $I_2$ to denote the two different collections of sets. Now for any $b_0>0$, let $B = \{f < b_0\}\cup \{\tilde f < b_0\}$. We divide the KL-divergence between $f$ and $\tilde f$ into three regions and bound them accordingly.
\begin{align*}
D_{KL}(f\|\tilde f) = \int_\Omega f\log \frac{f}{\tilde f} &= \int_{B} f\log \frac{f}{\tilde f} + \int_{B^c\cap \big(\bigcup I_1\big)} f\log \frac{f}{\tilde f} + \int_{B^c\cap \big(\bigcup I_2\big)} f\log \frac{f}{\tilde f}\\
& = M_1 + M_2 + M_3.
\end{align*}

We first look at $M_1$.
Because $\tilde f$ is a block-valued function, $\{\tilde f < b_0\}$ must be the union of all the $\tilde A_k$ that satisfies $\int_{\tilde A_k} f(\theta)d\theta\leq b_0|\tilde A_k|$. Therefore, we have
\begin{align*}
\int_{\tilde f < b_0} f(\theta)d\theta = \sum_{\tilde A_k: ~\int_{\tilde A_k} f(\theta)d\theta\leq b_0|\tilde A_k|} \int_{\tilde A_k} f(\theta)d\theta \leq \sum_{\tilde A_k: ~\int_{\tilde A_k} f(\theta)d\theta \leq b_0|\tilde A_k|} b_0 |\tilde A_k| \leq b_0.
\end{align*}
Therefore, we have 
\begin{align*}
\int_B f(\theta)d\theta \leq \int_{\tilde f< b_0}f(\theta)d\theta + \int_{f<b_0}f(\theta)d\theta = b_0 + b_0|\Omega| = 2b_0.
\end{align*}
Because $\tilde f \geq \min_k P(A_k)/|A_k| \geq P(A_k) \geq \rho$, we have
\begin{align*}
M_1 = \int_{B} f\log \frac{f}{\tilde f} \leq \int_B f(\theta)\log \frac{b_0}{\rho} \leq 2b_0\big|\log\frac{b_0}{\rho}\big|.
\end{align*}
\par
Next, we look at $M_2$. It is clear that 
\begin{align*}
\int_{B^c\cap \big(\bigcup I_1\big)} f(\theta)d\theta \leq  \int_{\big(\bigcup I_1\big)} f(\theta)d\theta \leq card(I_1)2\rho,
\end{align*}
and hence we have
\begin{align*}
M_2 = \int_{B^c\cap \big(\bigcup I_1\big)} f\log \frac{f}{\tilde f} \leq \int_{B^c\cap \big(\bigcup I_1\big)} f(\theta) \log\frac{D}{b_0}\leq card(I_1)2\rho\big|\log\frac{D}{b_0}\big|\leq 2K\rho\big|\log\frac{D}{b_0}\big|.
\end{align*}
\par
Now for $M_3$, we first use the inequality that $\log x \leq x - 1$ for any $x > 0$,
\begin{align*}
M_3 = \int_{B^c\cap \big(\bigcup I_2\big)} f\log \frac{f}{\tilde f} \leq \int_{B^c\cap \big(\bigcup I_2\big)} f \bigg(\frac{f}{\tilde f} - 1\bigg)\leq \int_{B^c\cap \big(\bigcup I_2\big)} \frac{f}{\tilde f} (f - \tilde f)\leq \frac{D}{b_0}\int_{\bigcup I_2} |f - \tilde f|.
\end{align*}
Using the mean value theorem for integration, we have $\int_{\tilde A_k} f(\theta)/|\tilde A_k| = f(\theta_0)$ for some $\theta_0 \in \tilde A_k$. Also because $f\in C^1(\Omega)$, $\|f'\|_\infty$ is bounded, i.e., there exists some constant $L$ such that
$|f(x_1) - f(x_2)| \leq L\sum_{j = 1}^p |x_{1j} - x_{2j}|$.
Therefore, we have
\begin{align*}
\int_{\bigcup I_2} |f - \tilde f| = \sum_{\tilde A_k \in I_2} \int_{\tilde A_k} |f(\theta) - \tilde f(\theta)| = \sum_{\tilde A_k \in I_2} \int_{\tilde A_k} |f(\theta) - f(\theta_0)| \leq \sum_{\tilde A_k \in I_2} 4pLk^{-\frac{1}{p}}|\tilde A_k| \leq 4pLk^{-\frac{1}{p}},
\end{align*}
and thus
\begin{align*}
M_3 \leq \frac{4pLD}{b_0}K^{-\frac{1}{p}}.
\end{align*}
Putting all pieces together we have
\begin{align*}
D_{KL}(f\|\tilde f) \leq 2b_0\big|\log\frac{b_0}{\rho}\big| + 2K\rho\big|\log\frac{D}{b_0}\big| + \frac{4pLD}{b_0}K^{-\frac{1}{p}}.
\end{align*}
Now taking $b_0 = K^{-1/(2p)}$ and $\rho = K^{-1-1/(2p)}$, we have
\begin{align*}
D_{KL}(f\|\tilde f) &\leq (\log K )K^{-\frac{1}{2p}} + 2(\log D + \frac{\log K}{2p})K^{-\frac{1}{2p}} + 4pLDK^{-\frac{1}{2p}}\\
&\leq (2\log K + 2\log D + 4pLD) K^{-\frac{1}{2p}}.
\end{align*}
Now defining $C_1 = 2\log D + 4pLD$ and $C_2 = 48\sqrt{p + 1}$ and combining with \eqref{eq:thm1.first}, we have
\begin{align*}
D_{KL}(f\|\hat f_{ML}) \leq (C_1+2\log K)K^{-\frac{1}{2p}} + C_2\max\big\{\log D, 2\log K\big\}\sqrt{\frac{K}{N}\log\bigg(\frac{3e N}{K}\bigg)\log\bigg(\frac{8}{\delta}\bigg)}.
\end{align*}
\end{proof}

\section*{Appendix C: Proof of Theorem \ref{thm:2}}
The KD-tree $\hat f_{KD}$ always cuts at the empirical median for different dimensions, aiming to approximate the true density by equal probability partitioning. For $\hat f_{KD}$ we have the following result.
\begin{theorem}
For any $\varepsilon > 0$, define $r_\varepsilon = \log_2\bigg(1 + \frac{1}{2 + 3L/\varepsilon}\bigg)$. For any $\delta > 0$, if $N > 32e^2(\log K)^2K\log (2K/\delta)$, then with probability at least $1 - \delta$, we have
\begin{align*}
\|\hat f_{KD} - f_{KD}\|_1\leq \varepsilon + pLK^{-\frac{r_\varepsilon}{p}} + 4e\log K\sqrt{\frac{2K}{N}\log\bigg(\frac{2K}{\delta}\bigg)}.
\end{align*}
If the function is further lower bounded by some constant $b_0 > 0$, we can then obtain an upper bound on the KL-divergence. Define $r_{b_0} = \log_2\bigg(1 + \frac{1}{2 + 3L/b_0}\bigg)$ and we have
\begin{align*}
D_{KL}(f\|\hat f_{KD}) \leq \frac{pLD}{b_0}K^{-\frac{r_{b_0}}{p}} + 8e\log K\sqrt{\frac{2K}{N}\log\bigg(\frac{2K}{\delta}\bigg)}.
\end{align*}

\end{theorem}
When there are multiple functions and the median partition is performed on pooled data, the partition might not happen at the empirical median on each subset. However, as long as the partition quantiles are upper and lower bounded by $\alpha$ and $1 - \alpha$ for some $\alpha \in [1/2, 1)$, we can establish similar theory as Theorem \ref{thm:2}.

\begin{corollary}\label{cor:2} Assume we instead partition at different quantiles that are upper and lower bounded by $\alpha$ and $1 - \alpha$ for some $\alpha \in [1/2, 1)$.
Define $r_\varepsilon = \log_2\bigg(1 + \frac{(1- \alpha)}{2\alpha + 3L/\varepsilon + 1}\bigg)$ and $r_{b_0} = \log_2\bigg(1 + \frac{(1- \alpha)}{2\alpha + 3L/b_0 + 1}\bigg)$. For any $\delta > 0$, if $N > \frac{12e^2}{(1 - \alpha)^2}K(\log K)^2\log(K/\delta)$, then with probability at least $1 - \delta$ we have
\begin{align*}
\|\hat f_{KD} - f_{KD}\|_1\leq \varepsilon + pLK^{-\frac{r_\varepsilon}{p}} + \frac{2e\log K}{1 - \alpha}K^{1 - \log_2\alpha^{-1}} \sqrt{\frac{3K}{N}\log\bigg(\frac{2K}{\delta}\bigg)}
\end{align*}
and if the function is lower bounded by $b_0$, then we have
\begin{align*}
D_{KL}(f\|\hat f_{KD}) \leq \frac{pLD}{b_0}K^{-\frac{r_{b_0}}{p}} + \frac{4e\log K}{1 - \alpha}K^{1 - \log_2\alpha^{-1}} \sqrt{\frac{3K}{N}\log\bigg(\frac{2K}{\delta}\bigg)}.
\end{align*}
\end{corollary}

Following the same argument for Theorem \ref{thm:1}, we will prove Theorem \ref{thm:2} by assuming a predetermined order of partition (such as $\{1, 2, 3, \cdots, p, 1, 2, \cdots\}$) for simplicity, though in the actual precedure, the dimensions are selected completely at random. We need the following two lemmas to prove Theorem \ref{thm:2}. Let $f_{KD}$ have the same partition as $\hat f_{KD}$ but with function value replaced by the true probability on each region divided by the area, i.e., $f_{KD} = \sum_{A_k}\frac{\int_{A_k} f(\theta)d\theta}{|A_k|}\tmop{1}_{A_k}(\theta)$.
\begin{lemma}\label{lemma:KD1}
With $f_{KD}$ defined above, for any $\delta > 0$, if $N > 32e^2(\log K)^2 K \log (K/\delta)$, then with probability at least $1 - \delta$, we have
\begin{align*}
\|\hat f_{KD} - f_{KD}\|_1\leq 4e\log K\sqrt{\frac{2K}{N}\log\bigg(\frac{K}{\delta}\bigg)}
\end{align*}
and
\begin{align*}
D_{KL}(f\|\hat f_{KD}) \leq D_{KL}(f\|f_{KD}) + 8e\log K\sqrt{\frac{2K}{N}\log\bigg(\frac{K}{\delta}\bigg)}.
\end{align*}
If we instead partition at some different quantiles, which are upper bounded by $\alpha$ and lower bounded by $1 - \alpha$ for some $\alpha\in [1/2, 1)$, then for any $\delta > 0$, if $N > \frac{12e^2}{(1 - \alpha)^2}K(\log K)^2 \log(K/\delta)$, with probability at least $1 - \delta$ we have
\begin{align*}
\|\hat f_{KD} - f_{KD}\|_1 \leq \frac{2e\log K}{1 - \alpha}K^{1 - \log_2\alpha^{-1}} \sqrt{\frac{3K}{N}\log\bigg(\frac{K}{\delta}\bigg)},
\end{align*}
and
\begin{align*}
D_{KL}(f\|\hat f_{KD}) \leq D_{KL}(f\|f_{KD}) + \frac{4e\log K}{1 - \alpha}K^{1 - \log_2\alpha^{-1}} \sqrt{\frac{3K}{N}\log\bigg(\frac{K}{\delta}\bigg)}.
\end{align*}
\end{lemma}
\begin{proof}
The proof is based on how close the data median is to the true median. Suppose there are $N_i$ points in the current region, condition on this region, and partition it into two regions $\hat A_1$ and $\hat A_2$ by cutting at the data median point $\hat M_i$. Denote the true median by $M_i$ and two anchor points $M_i - \epsilon_1$, $M_i + \epsilon_2$ such that $P(X \leq M_i - \epsilon_1) = 1/2 - t$ and $P(X \leq M_i + \epsilon_2) = 1/2 + t$ for some $0 < t < 1/2$. By Chernoff's inequality we have
\begin{align*}
P(\hat M_i \leq M_i - \epsilon_1) \leq \exp\bigg\{-\frac{t^2N_i}{1 + 2t}\bigg\}
\end{align*}
and
\begin{align*}
P(\hat M_i \geq M_i +\epsilon_2) \leq \exp\bigg\{-\frac{t^2N_i}{1 + 2t}\bigg\}.
\end{align*}
The above two inequalities indicate that with high probability $\hat M_i$ is within the interval $(M_i - \epsilon_1, M_i + \epsilon_2)$. Therefore, the probabilities on $A_1$ and $A_2$ also satisfy that
\begin{align}
P\bigg(\big|\mathbbm{E}\bm{1}_{A_i} - \frac{1}{2}\big|\geq t\bigg)\leq \exp\bigg\{-\frac{t^2N_i}{1 + 2t}\bigg\}.\label{eq:deviate}
\end{align}
Now consider the $K$ regions of $\hat f_{KD}$ and $f_{KD}$. Each partition will bring an error of at most $1/2 + t$ to the estimation of the region probability. Therefore, assuming $K\in [2^d + 1, 2^{d + 1})$ we have for each region $A_k$ ($A_k$ is a random variable) that
\begin{align*}
\bigg(\frac{1}{2} - t\bigg)^d \leq \int_{A_k}f(\theta) d\theta \leq \bigg(\frac{1}{2} + t\bigg)^d
\end{align*}
if $n_k /N = 1/2^d$, or
\begin{align*}
\bigg(\frac{1}{2} - t\bigg)^{d+1} \leq \int_{A_k}f(\theta) d\theta \leq \bigg(\frac{1}{2} + t\bigg)^{d+1},
\end{align*}
if $n_k/N = 1/2^{d+1}$. Notice that for all iterations before the current partition, we always have $N_i \geq N/K$. Therefore the probability is guaranteed to be greater than $1 - K\exp\bigg\{-\frac{t^2N/K}{1 + 2t}\bigg\}$. The above result indicates that
\begin{align}
\notag\max_{A_k} \bigg|\int_{A_k} f_{KD}(\theta) &- \int_{A_k} \hat f_{KD}(\theta)\bigg| \leq \max\bigg\{\bigg(\frac{1}{2} + t\bigg)^{d+1} - \bigg(\frac{1}{2}\bigg)^{d+1}, -\bigg(\frac{1}{2} - t\bigg)^{d+1} + \bigg(\frac{1}{2}\bigg)^{d+1}\bigg\} \\
\notag & =\bigg(\frac{1}{2}\bigg)^{d+1} \max\bigg\{ (1 + 2t)^{d+1} - 1, 1 - (1 - 2t)^{d+1} \bigg\}\\
\notag &= \bigg(\frac{1}{2}\bigg)^{d+1}\bigg((d+1)(1 + 2\tilde t)^d 2t\bigg),
\end{align}
where $\tilde t\in (0, t)$. So if $t < 1/(2d)$, then $(1 + 2\tilde t)^d \leq (1 + 1/d)^d < e$ and we have
\begin{align}
\max_{A_k} \bigg|\int_{A_k} f_{KD}(\theta) &- \int_{A_k} \hat f_{KD}(\theta)\bigg|\leq 2(d+1)e\bigg(\frac{1}{2}\bigg)^{d+1} t\leq \frac{4et\log K}{K}.\label{eq:dif1}
\end{align}
This result implies that the total variation distance satisfies that
\begin{align*}
\|\hat f_{KD} - f_{KD}\|_1 = \sum_k \int_{A_k} |\hat f_{KD}(\theta) - f_{KD}(\theta)| = \sum_k \bigg|\int_{A_k} \hat f_{KD}(\theta) - \int_{A_k} f_{KD}(\theta)\bigg|\leq 4et\log K.
\end{align*}
Similarly, one can also prove that
\begin{align*}
\max_{A_k} \bigg|\frac{\int_{A_k} f_{KD}(\theta)}{\int_{A_k} \hat f_{KD}(\theta)} - 1\bigg|\leq \max_{A_k} \bigg|\frac{\int_{A_k} f_{KD}(\theta) - \int_{A_k} \hat f_{KD}(\theta)}{\int_{A_k} \hat f_{KD}(\theta)}\bigg| \leq 4et\log K.
\end{align*}
Denote $\int_{A_k} f(\theta) = \int_{A_k} f_{KD}(\theta)$ by $P(A_k)$ and $\int_{A_K} \hat f_{KD} f(\theta)$ by $\hat P(A_k)$. The KL-divergence can then be computed as
\begin{align*}
D_{KL}(f\|\hat f_{KD}) &- D_{KL}(f\|f_{KD}) = \sum_{A_k}\int_{A_k}f(\theta) (\log f_{KD}(\theta) - \log \hat f_{KD}(\theta))\\
&=\sum_{A_k} \int_{A_k}f(\theta) \bigg(\log \int_{A_k}f(\theta) - \log \int_{A_k} \hat f_{KD}(\theta)\bigg)\\
&=\sum_{A_k} P(A_k)\bigg(\log \frac{P(A_k)}{\hat P(A_k)} \bigg)\leq \sum_{A_k} P(A_k)\bigg(\frac{P(A_k)}{\hat P(A_k)}  - 1\bigg)\\
&=\sum_{A_k} \frac{P(A_k)}{\hat P(A_k)}\bigg(P(A_k) - \hat P(A_k)\bigg)\\
&\leq (1 + 4et\log K) 4et\log K \leq 8et\log K,
\end{align*}
as long as $t < \min\bigg\{\frac{1}{4e\log K}, \frac{1}{2\log_2 K}\bigg\}$, with probability at least $1 - K\exp\bigg\{-\frac{t^2N}{2K}\bigg\}$. Consequently, for any $\delta > 0$, if $N > 32e^2K(\log K)^2 \log (K/\delta)$, then with probability at least $1 - \delta$, we have
\begin{align*}
\|\hat f_{KD} - f_{KD}\|_1\leq 4e\log K\sqrt{\frac{2K}{N}\log\bigg(\frac{K}{\delta}\bigg)}
\end{align*}
and 
\begin{align*}
D_{KL}(f\|\hat f_{KD}) \leq D_{KL}(f\|f_{KD}) + 8e\log K\sqrt{\frac{2K}{N}\log\bigg(\frac{K}{\delta}\bigg)}.
\end{align*}
When the partition occurs at a different quantile, which is assumed to be $\alpha_i$ for iteration $i$, we have
\begin{align*}
\prod_{i = 1}^{d+1} (\alpha_i' - t) \leq \int_{A_k}f(\theta) d\theta \leq \prod_{i = 1}^{d+1} (\alpha_i' + t),
\end{align*}
where $\alpha_i' = \alpha_i$ if $A_k$ takes the region containing smaller data values and $\alpha_i' = 1 - \alpha_i$ if $A_k$ takes the other half. First, \eqref{eq:deviate} can be updated as
\begin{align}
P\bigg(\big|\mathbbm{E}\tmop{1}_{A_i} - \alpha_i'\big|\geq t\bigg)\leq \exp\bigg\{-\frac{t^2N_i}{3\alpha}\bigg\}\leq \exp\bigg\{-\frac{t^2N_i}{3}\bigg\}\label{eq:deviate2},
\end{align}
if $t < 1 - \alpha$.
Then we can bound the difference between $\int_{A_k}f(\theta)$ and $\int_{A_k}\hat f_{KD}(\theta)$ as
\begin{align*}
\max_{A_k} \bigg|\int_{A_k} f_{KD}(\theta) &- \int_{A_k} \hat f_{KD}(\theta)\bigg| \leq \max\bigg\{\prod_{i = 1}^{d+1} (\alpha_i' + t) - \prod_{i = 1}^{d+1} \alpha_i',~ -\prod_{i = 1}^{d+1} (\alpha_i' - t) + \prod_{i = 1}^{d+1} \alpha_i'\bigg\}\\
&\leq \max\bigg[\prod_{i = 1}^{d+1} \alpha_i'\bigg\{\prod_{i = 1}^{d+1} \bigg(1 + \frac{t}{\alpha_i'}\bigg) - 1\bigg\}, ~\prod_{i = 1}^{d+1} \alpha_i'\bigg\{1 - \prod_{i = 1}^{d+1} \bigg(1 - \frac{t}{\alpha_i'}\bigg)\bigg\}\bigg]\\
&\leq \prod_{i = 1}^{d+1} \alpha_i'\cdot (d+1)\bigg(1 + \frac{\tilde t}{1 - \alpha}\bigg)^{d}\frac{t}{1-\alpha},
\end{align*}
where $\tilde t \in (0, t)$. Thus if $t < (1 - \alpha)/d$, then we have
\begin{align*}
\max_{A_k} \bigg|\int_{A_k} f_{KD}(\theta) &- \int_{A_k} \hat f_{KD}(\theta)\bigg|\leq \frac{e(d+1)t\alpha^{d+1}}{1-\alpha} \leq \frac{2et\log K}{(1 - \alpha)K^{\log_2\alpha^{-1}}},
\end{align*}
and
\begin{align*}
\max_{A_k} \bigg|\frac{\int_{A_k} f_{KD}(\theta)}{\int_{A_k} \hat f_{KD}(\theta)} - 1\bigg|\leq \frac{2et\log K}{1 - \alpha}.
\end{align*}
The total variation distance follows
\begin{align*}
\|\hat f_{KD} - f_{KD}\|_1 \leq K\cdot \frac{2et\log K}{(1 - \alpha)K^{\log_2\alpha^{-1}}} = \frac{2et\log K}{1 - \alpha}K^{1 - \log_2 \alpha^{-1}},
\end{align*}
and the KL-divergence follows
\begin{align*}
D_{KL}(f\|\hat f_{KD}) &- D_{KL}(f\|f_{KD}) \leq \bigg(1 + \frac{2et\log K}{1 - \alpha}\bigg)\frac{2et\log K}{1 - \alpha}K^{1 - \log_2 \alpha^{-1}}\leq \frac{4et\log K}{1 - \alpha}K^{1 - \log_2 \alpha^{-1}},
\end{align*}
if $t < 1/(2e(1 - \alpha)\log K)$.
Consequently, for any $\delta > 0$, if $N > \frac{12e^2}{(1 - \alpha)^2}K(\log K)^2 \log(K/\delta)$, then with probability at least $1 - \delta$, we have
\begin{align*}
\|\hat f_{KD} - f_{KD}\|_1 \leq \frac{2e\log K}{1 - \alpha}K^{1 - \log_2\alpha^{-1}} \sqrt{\frac{3K}{N}\log\bigg(\frac{K}{\delta}\bigg)},
\end{align*}
and
\begin{align*}
D_{KL}(f\|\hat f_{KD}) \leq D_{KL}(f\|f_{KD}) + \frac{4e\log K}{1 - \alpha}K^{1 - \log_2\alpha^{-1}} \sqrt{\frac{3K}{N}\log\bigg(\frac{K}{\delta}\bigg)}.
\end{align*}
\end{proof}

Our next result is to bound the distance between $f_{KD}$ and the true density $f$. Again, the proof depends on the control of the smallest value of $f$ and the longest edge of every block. One issue now is that each partition might not happen at the midpoint, but it should not deviate from the midpoint too much given the bound on the $f'$, i.e., we have the following proposition.
\begin{Prop}\label{prop:1}
Assume we aim to partition an edge of length $h$ (on dimension $q$) of a rectangular region $A$, which has a probability of $P$ and an area of $|A|$. We distinguish the resulting two regions as the left and the right region and the corresponding edges (on dimension $q$) as $h_\text{left}$ and $h_\text{right}$ (i.e., $h_\text{left} + h_\text{right} = h$). Suppose the partition ensures that the left region has probability of $\gamma P$, where $\gamma \geq 1/2$. If $\|f'\|_\infty \leq L$, then the longer edge $h^* = \max\{h_\text{left}, h_\text{right}\}$ satisfies that
\begin{align*}
\frac{h^*}{h} \leq 1 - \frac{1 - \gamma}{1 + Lh\frac{|A|}{P}}.
\end{align*}
\end{Prop}

\begin{proof}
It suffices to bound $h_\text{left}$ as $\gamma > 1 - \gamma$. Let $g(t) = \int_{x: (t, x)\in A} f(t, x)$, where $t$ represents the variable of dimension $q$ and $x$ stands for the other dimensions. We then have $\int_{t: h} g(t) = P$, $\int_{x: (t, x)\in A} 1dx = |A|/h$ and 
\begin{align*}
|g(t_1) - g(t_2)| \leq \int_{x: (t, x)} (f(t_1, x) - f(t_2, x)) \leq L|t_1 - t_2||A|/h.
\end{align*}
Therefore, using the mean value theorem for the integration, we know that
\begin{align*}
\bigg|\frac{\int_{t : h_\text{left}} g(t)}{h_\text{left}} - \frac{\int_{t: h_\text{right}} g(t)}{h_\text{right}}\bigg|\leq L|A|,
\end{align*}
which implies that
\begin{align*}
\bigg|\frac{\gamma h}{h_\text{left}} - \frac{(1-\gamma) h}{h_\text{right}}\bigg|\leq \frac{L|A|h}{P}.
\end{align*}
Now if we solve the following inequality
\begin{align*}
|\gamma/a - (1 - \gamma)/b|\leq c\quad\mbox{and}\quad a + b = 1, a\geq 0, b\geq 0,
\end{align*}
with some simple algebra we can get
\begin{align*}
a \leq 1 - \frac{1 - \gamma}{1 + c}.
\end{align*}
Plug in the corresponding value, we have
\begin{align*}
\frac{h_\text{left}}{h} \leq 1 - \frac{1 - \gamma}{1 + Lh\frac{|A|}{P}}.
\end{align*}
\end{proof}

With Proposition \ref{prop:1}, we can now obtain the upper bound for $\|f - f_{KD}\|_1$ and $D_{KL}(f\|f_{KD})$.
\begin{lemma}\label{lemma:KD2}
For any $\varepsilon > 0$, define $r_\varepsilon = \log_2\bigg(1 + \frac{1}{2 + 3L/\varepsilon}\bigg)$. If $N\geq 72K\log (K/\delta)$ for any $\delta > 0$, then with probability at least $1 - \delta$, we have
\begin{align*}
\|f - f_{KD}\|_1\leq \varepsilon + pLK^{-\frac{r_\varepsilon}{p}}.
\end{align*}
If the $f$ is further lower bounded by some $b_0 > 0$, the KL-divergence can be bounded as
\begin{align*}
D_{KL}(f\|f_{KD}) \leq \frac{pLD}{b_0}K^{-\frac{r_{b_0}}{p}},
\end{align*}
where $r_{b_0} = \log_2\bigg(1 + \frac{1}{2 + 3L/b_0}\bigg)$.
\par
Now suppose we instead partition at different quantiles, upper and lower bounded by $\alpha$ and $1 - \alpha$ for some $\alpha \in (1/2, 1)$. For any $\delta > 0$, if $N \geq \frac{27}{(1 - \alpha)^2}K\log(K/\delta)$ then the above two bounds hold with different $r_\varepsilon$ and $r_{b_0}$ as
\begin{align*}
r_\varepsilon = \log_2\bigg(1 + \frac{(1- \alpha)}{2\alpha + 3L/\varepsilon + 1}\bigg)\quad\mbox{and}\quad r_{b_0} = \log_2\bigg(1 + \frac{(1- \alpha)}{2\alpha + 3L/b_0 + 1}\bigg).
\end{align*}
\end{lemma}

\begin{proof}
The proof for the total variation distance follows similarly as Theorem \ref{thm:1}. For any $\varepsilon > 0$, we consider $B = \{f_{KD} < \varepsilon/2\}$. We then partition the total variation distance formula into two parts
\begin{align*}
\|f_{KD} - f\|_1 = \int_B |f_{KD} - f| + \int_{B^c} |f_{KD} - f| = M_1 + M_2.
\end{align*}
It is straightforward to bound $M_1$. $B$ is a union of $A_k$'s which satisfies $\int_{A_k} f(\theta) \leq \varepsilon |A_k|/2$. Therefore,
\begin{align*}
M_1 \leq \int_B f + \int_B f_{KD} =\sum_{A_k: \cup A_k = B} 2\int_{A_k} f(\theta) \leq \varepsilon.
\end{align*}
Now for $M_2$, the usual analysis shows that our result depends on the longest edge of each block, i.e.,
\begin{align*}
M_2 = \int_{B^c} |f_{KD} - f| = \sum_{A_k: \cup A_k = B^c} \int_{A_k} |f_{KD} - f| \leq \sum_{A_k: \cup A_k = B^c} pLh^*_k|A_k| = pL|B^c|\max_{A_k} h^*_k \leq pL\max_{A_k} h^*_k,
\end{align*}
where $h^*_k$ is the longest edge of each block contained in $B^c$. Now using Proposition \ref{prop:1}, we know for iteration $i$ the partitioned edge at each block follows
\begin{align*}
h_i \leq \bigg(1 - \frac{1 - \gamma}{1 + Lh\frac{|A|}{P}}\bigg)h_{i - 1} \leq \bigg(1 - \frac{1 - \gamma}{1 + L/\varepsilon}\bigg)h_{i - 1}.
\end{align*}
When $K \in (2^d, 2^{d+1}]$, each dimension receives $\lfloor d/p\rfloor $ stages of partitioning; therefore, we have for each block, the longest edge satisfies that
\begin{align*}
h^*\leq \bigg(1 - \frac{1 - \gamma}{1 + L/\varepsilon}\bigg)^{\frac{\log_2 K}{p}} \leq K^{-\frac{r_\varepsilon}{p}},
\end{align*}
where $r_\varepsilon = \log_2\bigg(1 + \frac{(1- \gamma)}{\gamma + L/\varepsilon}\bigg) $. This implies 
\begin{align*}
M_2\leq pLK^{-\frac{r_\varepsilon}{p}}\quad\mbox{and}\quad \|f_{KD} - f\|_1\leq \varepsilon + pLK^{-\frac{r_\varepsilon}{p}}.
\end{align*}
Now, according to \eqref{eq:deviate}, we know with probability at least $1 - K\exp\{-t^2N/(2K)\}$,
\begin{align*}
\gamma \leq \frac{1}{2} + t.
\end{align*}
Taking $t = 1/6$, we get 
\begin{align*}
r_\varepsilon = \log_2\bigg(1 + \frac{1}{2 + 3L/\varepsilon}\bigg),
\end{align*}
 with probability at least $1 - K\exp\{-N/(72K)\}$. So if $N > 72K\log(K/\delta)$, then the probability is at least $1 - \delta$. For the case when $\gamma = \alpha + t$, we choose $t = (1 - \alpha)/3$, then 
\begin{align*}
r_\varepsilon = \log_2\bigg(1 + \frac{(1- \alpha)}{2\alpha + 3L/\varepsilon + 1}\bigg)
\end{align*}
with probability at least $1 - \delta$ if $N > \frac{27}{(1 - \alpha)^2}K\log(K/\delta)$.
\par
For KL-divergence, if $f$ is lower bounded by some constant $b_0 > 0$, then we know that
\begin{align*}
D_{KL}(f\|f_{KD}) &= \int_\Omega f(\theta) \log \frac{f(\theta)}{f_{KD}(\theta)}\leq \int_\Omega f(\theta) \bigg(\frac{f(\theta)}{f_{KD}(\theta)} - 1\bigg)\\
&\leq \max_\theta\frac{f(\theta)}{f_{KD}(\theta)}\int_\Omega |f(\theta) - f_{KD}(\theta)|\leq\frac{D}{b_0}\|f - f_{KD}\|_1.
\end{align*}
Because $f$ and $f_{KD}$ are both lower bounded by $b_0$, we can follow the proof for $\|f - f_{KD}\|_1$ with $\varepsilon = b_0$ and ignore $M_1$. Thus we have
\begin{align*}
D_{KL}(f\|f_{KD}) \leq \frac{pLD}{b_0}K^{-\frac{r_{b_0}}{p}},
\end{align*}
where $r_{b_0} = \log_2\bigg(1 + \frac{(1- \gamma)}{\gamma + L/b_0}\bigg)$. Similarly, if we take $\gamma = 2/3$ and $N \geq 72K\log (K/\delta)$, then with probability at least $1 - \delta$, we have
\begin{align*}
r_{b_0} = \log_2\bigg(1 + \frac{1}{2 + 3L/b_0}\bigg).
\end{align*}
If we take $\gamma = (2\alpha + 1)/3$ and $N > \frac{27}{(1 - \alpha)^2}K\log(K/\delta)$, then with probability at least $1 - \delta$, we have
\begin{align*}
r_{b_0} = \log_2\bigg(1 + \frac{(1- \alpha)}{2\alpha + 3L/b_0 + 1}\bigg).
\end{align*}
\end{proof}

Theorem \ref{thm:2} and Corollary \ref{cor:2} follow directly from Lemma \ref{lemma:KD1} and \ref{lemma:KD2}.
\begin{proof}[\textbf{Proof of Theorem \ref{thm:2} and Corollary \ref{cor:2}}]
For any $\varepsilon > 0$, define $r_\varepsilon$ and $r_{b_0}$ as in Lemma \ref{lemma:KD2}. Thus, for any $\delta > 0$, if $N > 32e^2(\log K)^2K\log\frac{2K}{\delta}$, then with probability $1 - \delta/2$ we have
\begin{align*}
\|\hat f_{KD} - f_{KD}\|_1\leq 4e\log K\sqrt{\frac{2K}{N}\log\bigg(\frac{2K}{\delta}\bigg)}
\end{align*}
and
\begin{align*}
D_{KL}(f\|\hat f_{KD}) \leq D_{KL}(f\|f_{KD}) + 8e\log K\sqrt{\frac{2K}{N}\log\bigg(\frac{2K}{\delta}\bigg)}.
\end{align*}
Also, with probability $1 - \delta/2$ we have
\begin{align*}
\|f - f_{KD}\|_1\leq \varepsilon + pLK^{-\frac{r_\varepsilon}{p}},
\end{align*}
and
\begin{align*}
D_{KL}(f\|f_{KD}) \leq \frac{pLD}{b_0}K^{-\frac{r_{b_0}}{p}}.
\end{align*}
Putting the two equations together we have
\begin{align*}
\|\hat f_{KD} - f_{KD}\|_1\leq \varepsilon + pLK^{-\frac{r_\varepsilon}{p}} + 4e\log K\sqrt{\frac{2K}{N}\log\bigg(\frac{2K}{\delta}\bigg)},
\end{align*}
and
\begin{align*}
D_{KL}(f\|\hat f_{KD}) \leq \frac{pLD}{b_0}K^{-\frac{r_{b_0}}{p}} + 8e\log K\sqrt{\frac{2K}{N}\log\bigg(\frac{2K}{\delta}\bigg)}.
\end{align*}
Using the same argument on random quantiles, if $N > \frac{12e^2}{(1 - \alpha)^2}K(\log K)^2\log(2K/\delta)$, then with probability at least $1 - \delta$ we have 
\begin{align*}
\|\hat f_{KD} - f_{KD}\|_1\leq \varepsilon + pLK^{-\frac{r_\varepsilon}{p}} + \frac{2e\log K}{1 - \alpha}K^{1 - \log_2\alpha^{-1}} \sqrt{\frac{3K}{N}\log\bigg(\frac{2K}{\delta}\bigg)}
\end{align*}
and
\begin{align*}
D_{KL}(f\|\hat f_{KD}) \leq \frac{pLD}{b_0}K^{-\frac{r_{b_0}}{p}} + \frac{4e\log K}{1 - \alpha}K^{1 - \log_2\alpha^{-1}} \sqrt{\frac{3K}{N}\log\bigg(\frac{2K}{\delta}\bigg)},
\end{align*}
where $r_\varepsilon$ and $r_{b_0}$ are defined as
\begin{align*}
r_\varepsilon = \log_2\bigg(1 + \frac{(1- \alpha)}{2\alpha + 3L/\varepsilon + 1}\bigg)\quad\mbox{and}\quad r_{b_0} = \log_2\bigg(1 + \frac{(1- \alpha)}{2\alpha + 3L/b_0 + 1}\bigg).
\end{align*}
\end{proof}

\section*{Appendix E: Proof of Theorem \ref{thm:one} and \ref{thm:pairwise}}
\begin{lemma}\label{lemma:D}
Assume $\|f\|_\infty \leq D$. Under the same condition as Theorems \ref{thm:1} and \ref{thm:2}, if 
\begin{align*}
\sqrt N \geq 32c_0^{-1}\sqrt{2(p + 1)}K^{\frac{3}{2} + \frac{1}{2p}}\sqrt{\log\bigg(\frac{3e N}{K}\bigg)\log\bigg(\frac{8}{\delta}\bigg)},
\end{align*}
then we have $\|\hat f_{ML}\|_\infty \leq 2D$ and if 
\begin{align*}
N > 128e^2K(\log K)^2\log(K/\delta),
\end{align*}
we have $\|\hat f_{KD}\|_\infty \leq 2D$.
\end{lemma}
\begin{proof}
Assume $\|f\|_\infty \leq D$. We want to bound $\|\hat f_{ML}\|_\infty$ and $\|\hat f_{KD}\|_\infty$. Define 
$$\tilde f = \sum_{A_K} \frac{\int_{A_k}f(\theta)d\theta}{|A_k|} \tmop{1}_{A_k}(\theta),$$
which clearly satisfies $\tilde f\leq D$.
Notice that if there exists some $\epsilon$ such that
\begin{align*}
\max_{A_k} |P(A_k) - \hat P(A_k)| \leq \epsilon,
\end{align*}
where $P(A_k) = \mathbbm{E}\tmop{1}_{A_k}$ and $\hat P(A_k) = \mathbbm{\hat E}\tmop{1}_{A_k}$,
then we have
\begin{align*}
\|\tilde f - \hat f\|_\infty &= \max_{A_k} \bigg|\frac{P(A_k)}{|A_k|} - \frac{\hat P(A_k)}{|A_k|}\bigg| = \max_{A_k} \frac{1}{|A_k|} |P(A_k) - \hat P(A_k)|\\
&\leq \max_{ A_k} \frac{\epsilon \tilde f(\theta)}{P(A_k)} \leq \max_{A_k}\frac{\epsilon D}{\hat P(A_k) - \epsilon}.
\end{align*}
Now if we can pick $\epsilon = \min_{A_k}\hat P(A_k)/2$, then the upper bound becomes $2D$. We deduce the corresponding condition for ML-cut and KD-cut respectively. For maximum likelihood partition, plug in $\epsilon = K^{-1-1/(2p)}/2$ into \eqref{eq:concentration}. Under the condition of Theorem \ref{thm:1}, if 
\begin{align*}
\sqrt N \geq 32c_0^{-1}\sqrt{2(p + 1)}K^{\frac{3}{2} + \frac{1}{2p}}\sqrt{\log\bigg(\frac{3e N}{K}\bigg)\log\bigg(\frac{8}{\delta}\bigg)},
\end{align*}
then with probability at least $1 - \delta/2$, we have 
\begin{align*}
\|\hat f_{ML}\|_\infty \leq 2D.
\end{align*}
For median partition, choose $\epsilon = K^{-1}/2$ and apply \eqref{eq:dif1}. Under the condition of Theorem \ref{thm:2}, if
\begin{align*}
N > 128e^2K(\log K)^2\log(K/\delta),
\end{align*}
then with probability at least $1 - \delta$, we have
\begin{align*}
\|\hat f_{KD}\|_\infty \leq 2D.
\end{align*}
\end{proof}

\begin{proof}[\textbf{Proof of Theorem \ref{thm:one}}]
Assume the average total variation distance between $\hat f^{(i)}$ and $f^{(i)}$ is $\varepsilon$. It can be calculated directly as
\begin{align*}
 \int \left| \prod_{i \in I} f^{(i)} (\theta) - \prod_{ i \in I}^m \hat{f}^{(i)} (\theta) \right| d\theta & \leq \sum_{i = 1}^m \int | f^{(i)} (\theta) - \hat{f}^{(i)} (\theta) | \prod_{j = 1}^{i - 1} f^{(j)} (\theta) \prod_{l = i + 1}^m \hat{f}^{(l)}(\theta)\\
 & \leq (2D)^{m - 1} \sum_{i = 1}^m \int | f^{(i)} (\theta) - \hat{f}^{(i)} (\theta) | d\theta\\
 & \leq m(2D)^{m - 1}\varepsilon.
\end{align*}
Letting $\hat Z_I = \int\prod_{i \in I} f^{(i)}$, we have
\begin{align*}
|Z_I - \hat{Z_I} | = \left| \int \left(\prod_{i = 1}^m f^{(i)} (\theta) - \prod_{i = 1}^m \hat{f}^{(i)} (\theta) \right) d\theta \right| \leq\int\left| \prod_{i = 1}^m f^{(i)} (\theta) - \prod_{i = 1}^m \hat{f}^{(i)} (\theta)\right| d\theta \leq m(2D)^{m - 1}\varepsilon.
\end{align*}
Thus
\begin{align*}
\|f_I - \hat f_I\|_1 &= \int\left|\frac{1}{Z_I} \prod_{i = 1}^m f^{(i)}(\theta) - \frac{1}{\hat{Z_I}}\prod_{i = 1}^m \hat{f}^{(i)} (\theta) \right| d x = \int \left|\frac{\hat{Z_I} \prod_{i\in I} f^{(i)} - Z_I \prod_{i\in I} \hat{f}^{(i)}}{Z_I \hat{Z_I}} \right|d\theta\\
 & =  \int \left|\frac{\hat{Z_I} \prod_{i\in I} f^{(i)} - \hat{Z_I} \prod_{i\in I}\hat{f}^{(i)} + \hat{Z_I} \prod_{i\in I} \hat{f}^{(i)} - Z_I \prod_{i\in I} \hat{f}^{(i)}}{Z_I\hat{Z_I}} \right| d\theta\\
 & \leq \frac{1}{Z_I} \int \left| \prod_{i\in I} f^{(i)} - \prod_{i\in I}\hat{f}^{(i)} \right| d\theta + \frac{1}{Z_I} | \hat{Z_I} - Z_I |\\
 & \leq\frac{2}{Z_I} m(2D)^{m - 1}\varepsilon.
\end{align*}

\end{proof}

\begin{proof}[\textbf{Proof of Theorem \ref{thm:pairwise}}]
Assuming $m \in [2^s, 2^{s + 1})$, then after $s + 1$ iterations, we will obtain our final aggregated density. At iteration $l$, each true density is some aggregation of the original $m$ densities, which can be represented by $f_{I'}$, where $I'$ is the set of indices of the original densities. Let $\varepsilon_l^{(I_1,I_2)}$ be the total variation distance between the true density and the approximation for the pair $(I_1,I_2)$ caused by combining. For example, when $l = 1$, $I_1, I_2$ contain only a single element, i.e., $I_1 = \{i_1\}$ and $I_2 = \{i_2\}$. Recall that $C_0 = \max_{I''\subset I'\subseteq I} Z_{I''}Z_{I'\setminus I''}/Z_{I'}$, using the result from Theorem \ref{thm:one}, we have
\begin{align*}
\varepsilon_1^{(I_1,I_2)} = \bigg\|\frac{f^{(i_1)}f^{(i_2)}}{\int f^{(i_1)}f^{(i_2)}} - \frac{\hat f^{(i_1)}\hat f^{(i_2)}}{\int \hat f^{(i_1)}\hat f^{(i_2)}} \bigg\|_1 \leq \frac{2}{\int f^{(i_1)}f^{(i_2)}} 2D \varepsilon = \frac{2\int f^{(i_1)} \int f^{(i_2)}}{\int f^{(i_1)}f^{(i_2)}} 2D \varepsilon \leq 4C_0D\varepsilon.
\end{align*}
We prove the result by induction. Assuming we are currently at iteration $l+1$, and the paired two densities are $f_{I_1}$ and $f_{I_2}$ where $I_1, I_2\subseteq I$. By induction, the approximation obtained at iteration $l$ are $\hat f_{I_1}$ and $\hat f_{I_2}$ which satisfies that 
\begin{align*}
\|f_{I_1} - \hat f_{I_1}\|_1\leq (4C_0D)^{l}\varepsilon,\quad \|f_{I_2} - \hat f_{I_2}\|_1\leq (4C_0D)^{l}\varepsilon.
\end{align*}
Using Theorem \ref{thm:one} again, we have that
\begin{align*}
\varepsilon_{l + 1}^{(I_1, I_2)} &= \bigg\|\frac{f_{I_1}f_{I_2}}{\int f_{I_1}f_{I_2}} - \frac{\hat f_{I_1}\hat f_{I_2}}{\int \hat f_{I_1}\hat f_{I_2}}  \bigg\|_1 \leq \frac{2}{\int f_{I_1}f_{I_2}} (2D) \bigg\{\frac{\|f_{I_1} - \hat f_{I_1}\|_1 + \|f_{I_2} - \hat f_{I_2}\|_1}{2}\bigg\} \\
&\leq \frac{\int \prod_{i\in I_1} f^{(i)}\int \prod_{i\in I_2} f^{(i)}}{\int \prod_{i\in I_1\cup I_2}f^{(i)}} (4D)\cdot (4C_0D)^{l}\varepsilon \leq (4C_0D)^{l+1}\varepsilon
\end{align*}
Consequently, the final approximation satisfies that
\begin{align*}
\|f_I - \hat f_I\|_1 \leq (4C_0D)^{s+1}\varepsilon \leq (4C_0D)^{\log_2 m + 1}\varepsilon.
\end{align*}
\end{proof}

\section*{Appendix D: Supplement to Two Toy Examples}
\paragraph{Bimodal Example}
Figure \ref{fig:example-bimodal} compares the aggregated density of \emph{PART-KD/PART-ML} for several alternative combination schemes to the true density. This complements the results from one-stage combination with uniform block-wise distribution presented in Figure 1 of the main text. 

\begin{figure}[!htb]
\centering
\includegraphics[width=0.5\textwidth]{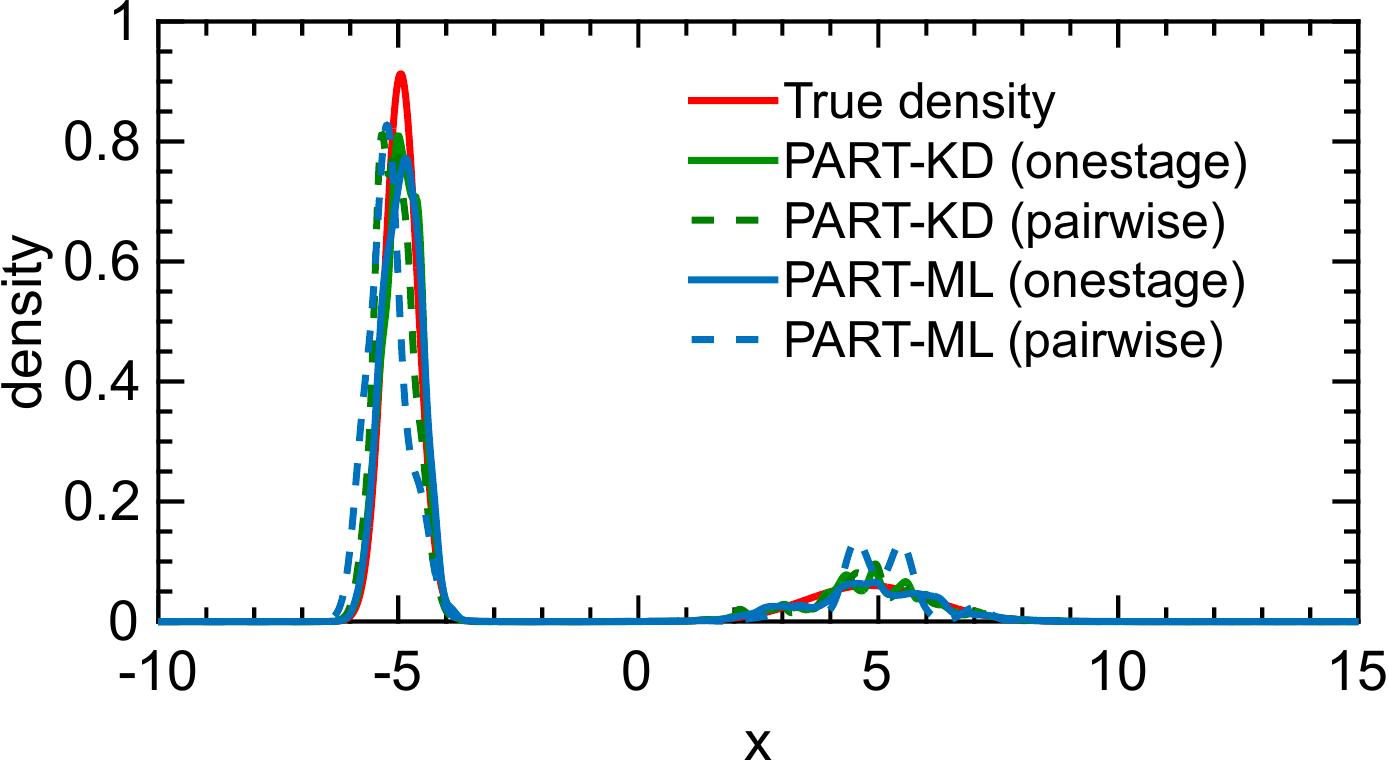}
\caption{Bimodal posterior combined from 10 subsets. The results from \emph{PART-KD/PART-ML} multiscale histograms are shown for (1) one-stage combination with local Gaussian smoothing (2) pairwise combination with local Gaussian smoothing.}
\label{fig:example-bimodal}
\end{figure}

\paragraph{Rare Bernoulli Example} The left panel of Figure \ref{fig:example-rare} shows additional results of posteriors aggregated from \emph{PART-KD/PART-ML} random tree ensemble with several alternative combination strategies, which complement the results presented in Figure 2 of the main text. All of the produced posteriors correctly locate the posterior mass despite the heterogeneity of subset posteriors. The fake ``ripples'' produced by pairwise ML aggregation are caused by local Gaussian smoothing. 

Also, the right panel of Figure \ref{fig:example-rare} shows that the posteriors produced by nonparametric and semiparametric methods miss the right scale. 

\begin{figure}[!htb]
\centering
\includegraphics[width=0.44\textwidth]{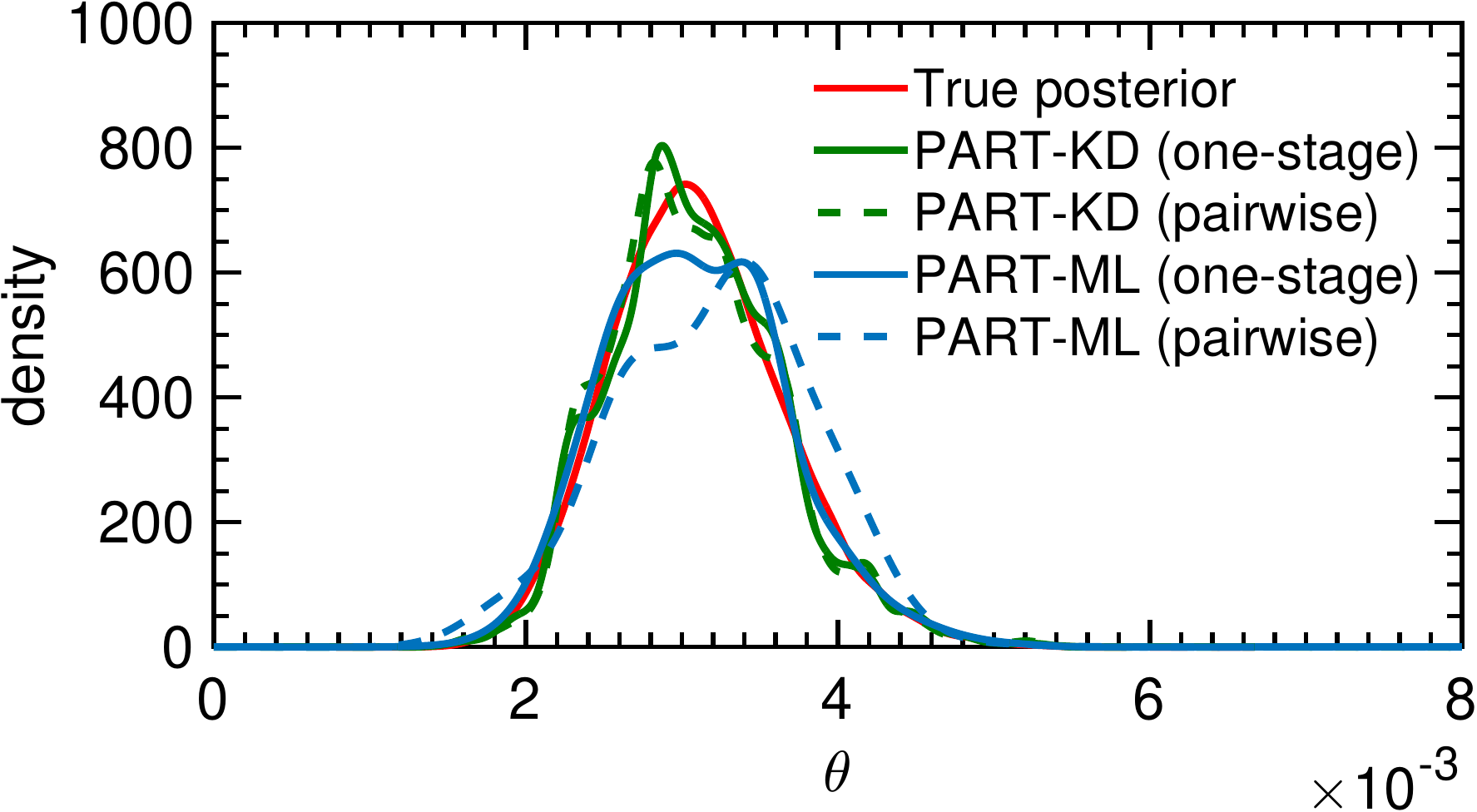}
\includegraphics[width=0.44\textwidth]{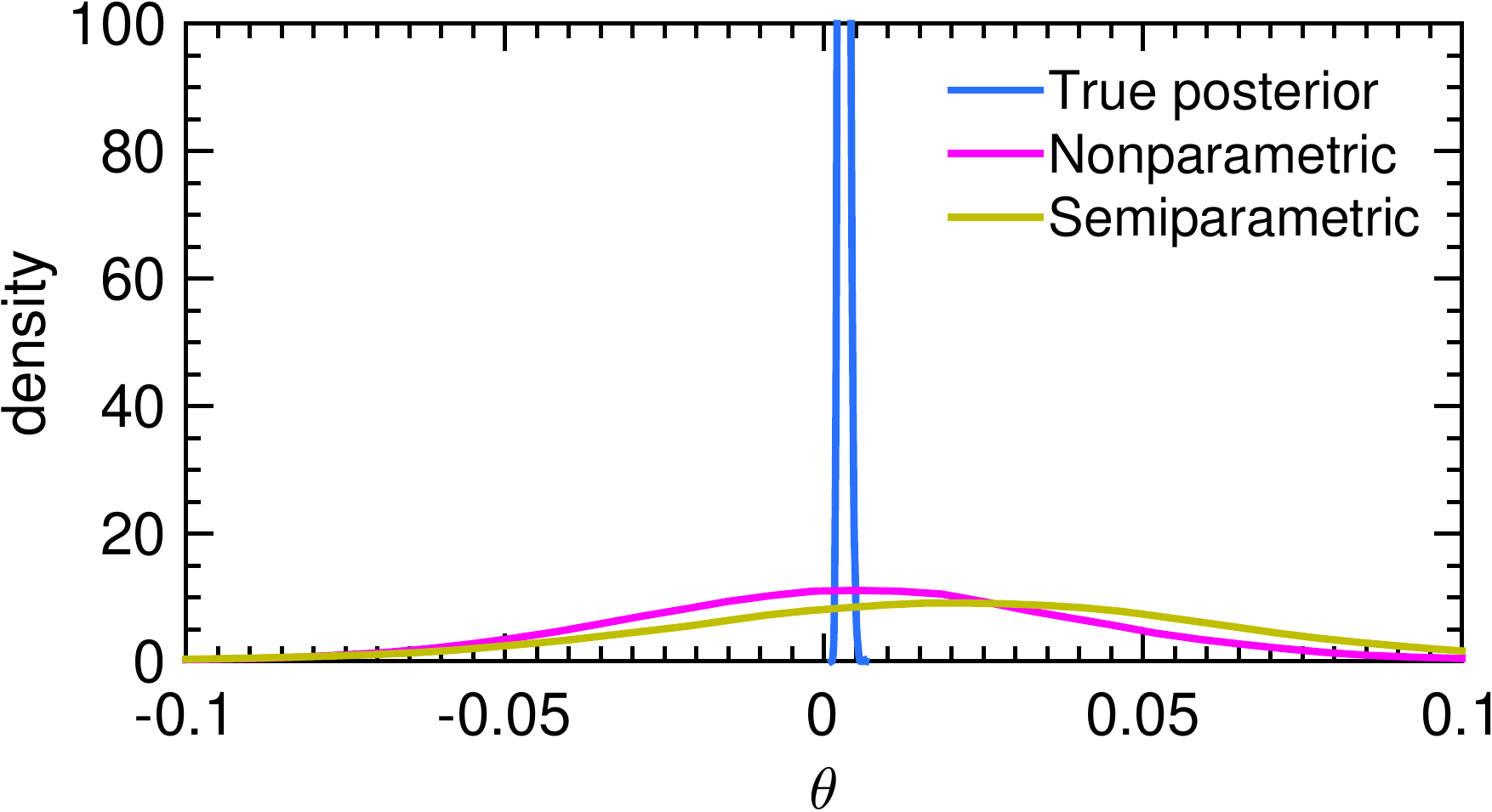}
\caption{Posterior of the probability $\theta$ of a rare event combined from $M=15$ subsets of independent Bernoulli trials. Left: the results from KD/ML multiscale histograms are shown for (1) one-stage combination with local Gaussian smoothing (2) pairwise combination with local Gaussian smoothing. Right: posterior aggregated from nonparametric and semiparametric methods. }
\label{fig:example-rare}
\end{figure}

\section*{Appendix F: Supplement to Bayesian Logistic Regression}
Figure \ref{fig:realdata} additionally plots the prediction accuracy against the length of subset chains supplied to the aggregation algorithms, for Bayesian logistic regression on two real datasets. For simplicity, the same number of posterior samples from all subset chains are aggregated, with the first 20\% discarded as burn-in. As a reference, we also show the result for running the full chain. As can be seen from Figure \ref{fig:realdata}, the performance of \emph{PART-KD/ML} agrees with that of the full chain as the number of posterior samples increase, validating the theoretical results presented in Theorem 1 and Theorem 2 in the main text.

\begin{figure}[!htb]
\centering
\includegraphics[width=0.47\textwidth]{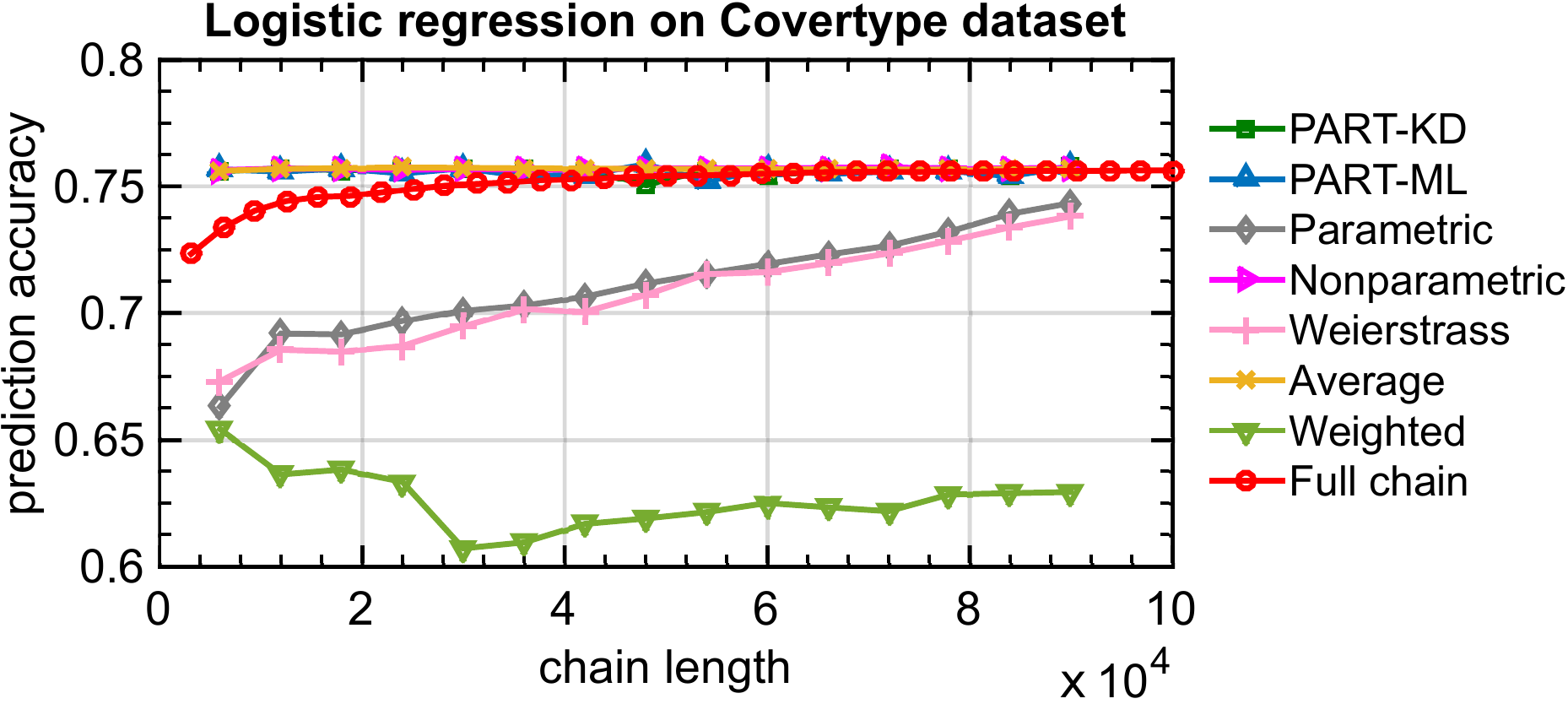}
\includegraphics[width=0.47\textwidth]{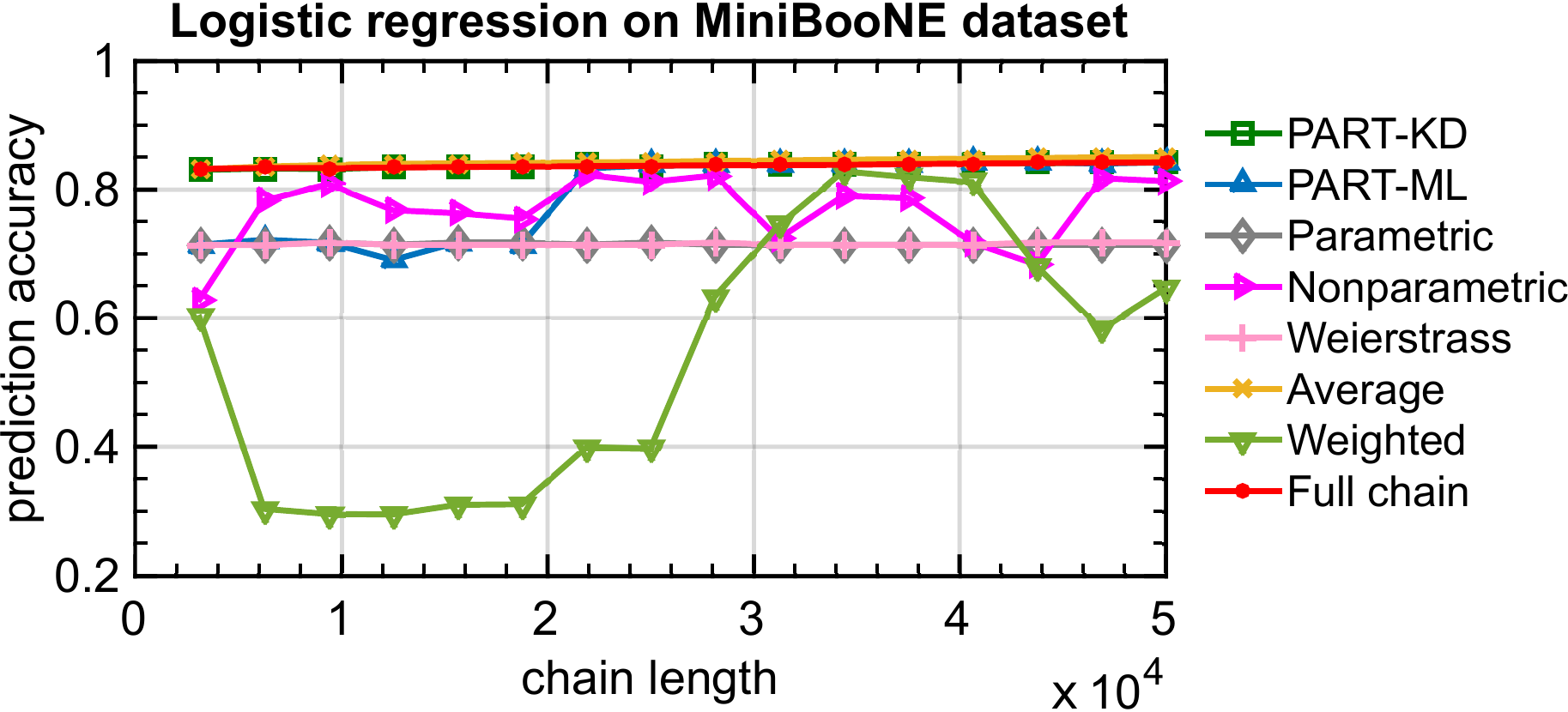}
\caption{Prediction accuracy versus the length of subset chains on the \emph{covertype} and the \emph{MiniBooNE} dataset.}
\label{fig:realdata}
\end{figure}

\newpage
\bibliographystyle{unsrt}
\bibliography{tree15}